\documentclass[sigconf]{aamas}

\usepackage{balance} %

\usepackage[utf8]{inputenc} %
\usepackage[T1]{fontenc}    %
\usepackage{hyperref}       %
\usepackage{url}            %
\usepackage{booktabs}       %
\usepackage{amsfonts}       %
\usepackage{nicefrac}       %
\usepackage{microtype}      %
\usepackage{xcolor}         %

\usepackage{algorithm}
\usepackage{algpseudocode}
\usepackage{setspace}

\usepackage{amsmath}
\usepackage{mathtools}
\usepackage{amsthm}
\usepackage{xspace}

\theoremstyle{plain}
\newtheorem{theorem}{Theorem}[section]

\newtheorem{lemma}[theorem]{Lemma}

\theoremstyle{definition}
\newtheorem{definition}[theorem]{Definition}

\theoremstyle{remark}

\usepackage{pifont}
\newcommand{\cmark}{\ding{51}}%
\newcommand{\xmark}{\ding{55}}%

\newcommand{\expt}{\mathop{\mathbb{E}}}

\newcommand{\bR}{\mathbb{R}}

\newcommand{\cA}{\mathcal{A}}
\newcommand{\cV}{\mathcal{V}}

\newcommand{\notp}{{\neg p}}
\newcommand{\neupl}{NeuPL\xspace}
\newcommand{\neupljpsro}{NeuPL-JPSRO\xspace}
\newcommand{\rps}{{\it rock-paper-scissors}\xspace}
\newcommand{\rws}{{\it running-with-scissors}\xspace}
\newcommand{\psro}{{\sc PSRO}\xspace}
\newcommand{\jpsro}{{\sc JPSRO}\xspace}

\newcommand{\eqdef}{\mathrel{\mathop:}=}

\DeclarePairedDelimiterX{\infdivx}[2]{\big[}{\big]}{%
  #1\;\delimsize|\delimsize|\;#2%
}
\newcommand{\kld}[2]{\ensuremath{D_{\textsc{KL}}\infdivx{#1}{#2}}\xspace}

\usepackage{listings}
\setcopyright{ifaamas}
\acmConference[AAMAS '24]{Proc.\@ of the 23rd International Conference
on Autonomous Agents and Multiagent Systems (AAMAS 2024)}{May 6 -- 10, 2024}
{Auckland, New Zealand}{N.~Alechina, V.~Dignum, M.~Dastani, J.S.~Sichman (eds.)}
\copyrightyear{2024}
\acmYear{2024}
\acmDOI{}
\acmPrice{}
\acmISBN{}

\acmSubmissionID{<<EasyChair submission id>>}

\title[Neural Population Learning beyond Symmetric Zero-sum Games]{Neural Population Learning beyond Symmetric Zero-sum Games}

\author{Siqi Liu}
\affiliation{
  \institution{Google DeepMind and\\University College London}
  \city{London}
  \country{United Kingdom}}
\email{liusiqi@google.com}
\orcid{https://orcid.org/0000-0001-6381-4552}

\author{Luke Marris}
\affiliation{
  \institution{Google DeepMind and\\University College London}
  \city{London}
  \country{United Kingdom}}
\email{marris@google.com}
\orcid{https://orcid.org/0009-0006-6300-9485}

\author{Marc Lanctot}
\affiliation{
  \institution{Google DeepMind}
  \city{Montréal}
  \country{Canada}}
\email{lanctot@google.com}
\orcid{https://orcid.org/0000-0003-4159-2519}

\author{Georgios Piliouras}
\affiliation{
  \institution{Google DeepMind}
  \city{London}
  \country{United Kingdom}}
\email{gpil@google.com}
\orcid{https://orcid.org/0000-0002-6236-3566}

\author{Joel Z. Leibo}
\affiliation{
  \institution{Google DeepMind}
  \city{London}
  \country{United Kingdom}}
\email{jzl@google.com}
\orcid{https://orcid.org/0000-0002-3153-916X}

\author{Nicolas Heess}
\affiliation{
  \institution{Google DeepMind}
  \city{London}
  \country{United Kingdom}}
\email{heess@google.com}
\orcid{https://orcid.org/0000-0001-7876-9256}

\begin{abstract}
We study computationally efficient methods for finding equilibria in n-player general-sum games, specifically ones that afford complex visuomotor skills. We show how existing methods would struggle in this setting, either computationally or in theory. We then introduce \neupljpsro, a neural population learning algorithm that benefits from transfer learning of skills and converges to a Coarse Correlated Equilibrium (CCE) of the game. We show empirical convergence in a suite of OpenSpiel games, validated rigorously by exact game solvers. We then deploy \neupljpsro to complex domains, where our approach enables adaptive coordination in a MuJoCo control domain and skill transfer in capture-the-flag. Our work shows that equilibrium convergent population learning can be implemented at scale and in generality, paving the way towards solving real-world games between heterogeneous players with mixed motives.
\end{abstract}


\newcommand{\BibTeX}{\rm B\kern-.05em{\sc i\kern-.025em b}\kern-.08em\TeX}

\begin{document}

\pagestyle{fancy}
\fancyhead{}

\maketitle

\section{Introduction}
\label{introduction}

Purely competitive, symmetric zero-sum games have proven to be popular testbeds for AI research since its early days \citep{samuel1967some, tesauro1995temporal, campbell2002deep, silver2017mastering, silver2018general, brown2018superhuman, vinyals2019grandmaster, perolat2022mastering}. Principled algorithms have been developed in this setting, with convergence guarantees to a Nash Equilibrium (NE, \cite{nash1951}) where players can be expected to win (or draw) against {\em any} opponent. One family of equilibrium convergent methods follows from Fictitious Play (FP, \cite{brown1951}) and Double Oracle (DO, \cite{mcmahan_planning_2003}). By learning a set of strategies each best-responding to a mixture over their predecessors, FP and DO converge to an NE even in cyclic games (e.g. \rps) where self-play would have made no progress. Policy-Space Response Oracle (\psro, \cite{lanctot2017unified}) extended similar guarantees to extensive-form (EF, \cite{kuhn1957efg}) games by constructing a normal-form (NF) metagame whose actions correspond to playing a deep reinforcement learning (RL) policy for an entire episode. Variations of this idea have led to competitive agents in games as complex as StarCraft \citep{vinyals2019grandmaster}, albeit at significant costs training and evaluating hundreds of independent deep RL agents.

While significant advances have been made in finding Nash Equilibria in symmetric zero-sum games, real-world interactions are often n-player general-sum --- between heterogeneous actors with mixed motives. Progress here has been more limited for a few reasons. Computationally, finding exact NE is intractable beyond two-player zero-sum games (i.e. PPAD-complete \citep{daskalakis2009complexity}). More importantly, NE describe an impoverished view of general-sum interactions as it forbids correlated action choices between players. This limitation is subtle but critical: consider a road junction, an NE can only suggest uncorrelated action choices for each driver when much improved outcomes could have been achieved by coordinating drivers with a trusted third-party (e.g. a traffic light). Similar general-sum interactions occur frequently in our society. Fair, mutually beneficial social norms often enable coordination and improve outcomes for all parties.

This observation motivated (Coarse) Correlated Equilibria ((C)CE, \cite{aumann1974subjectivity, moulin1978strategically}), an equilibrium solution concept that allows for coordinated actions between players, mediated by a correlation device that {\em rational}, {\em self-interested} players would find beneficial to follow (go on ``green'', wait on ``red''). (C)CE generalise NE, as they naturally reduce to NE if players are in pure competition and have no way to usefully coordinate. Unlike NE, (C)CE are computationally tractable too, as they can be formulated as a linear program (LP) even in the n-player general-sum setting. Algorithms that offer convergence to (C)CE have also received increased interest in recent years. Following similar iterative best-response arguments as PSRO, \cite{marris2021jpsroicml} proposed Joint PSRO (\jpsro), a population learning algorithm with convergence guarantees to a NF (C)CE in n-player general-sum EF games. Nevertheless, evidence of convergence to (C)CE has been limited to a few research games that can be solved analytically --- the costs of representing, training and evaluating a population of independent RL agents {\em for each player} quickly become intractable, especially in games that demand complex skills.

How can we bring game-theoretic algorithms such as (J)PSRO to real-world games in full generality and at scale? The central question here, we argue, is that of efficient policy representation. If human players can routinely combine and reuse different skills to develop new strategies in games such as Chess, Go and Poker, could artificial agents reuse fundamental skills such as locomotion, perception and memory across strategies too? Whereas such skills must be learned repeatedly at each iteration in (J)PSRO using independent RL, Neural Population Learning (\neupl, \cite{liu2022neupl, liu2022simplex}) transfers and refines such skills across all policies within the population. Conditioned on opponent priors, the same neural network represents diverse best-response policies and implements \psro in symmetric zero-sum games. This approach offers several advantages. The shared, probabilistic representation of opponent behaviours promotes skill transfer between strategies, allows for online adaptation under different opponent priors and reduces the computational costs of training a population of policies to be comparable to that of self-play. Despite strong empirical results in several test domains, \neupl remains limited in important ways. First, \neupl is restricted to symmetric zero-sum games, limiting its generality in practical applications. Second, the convergence guarantees of population learning algorithms with shared strategy representation requires further clarification, as we shall explain in this work. Relatedly, \neupl fell short of demonstrating empirical convergence to a NE in research games where convergence can be verified empirically.

We address these limitations with \neupljpsro, a scalable, equilibrium convergent algorithm for n-player general-sum games. We clarify the convergence guarantees of population learning algorithms with shared representation, before motivating fundamental departures from \neupl that ensure convergence in general games. Empirically, we show that \neupljpsro converges to a CCE in several OpenSpiel games at a rate of convergence comparable to exact \jpsro where solutions can be evaluated exactly using analytical game solvers. Lastly, we show that unlike \jpsro, \neupljpsro can be deployed efficiently in complex domains, demonstrating adaptive, coordinated control in MuJoCo control domains and transfer learning of skills in the team strategy game of capture-the-flag that requires spatiotemporal reasoning from partial, visual observations. 

Additional experimental details and proofs are available in appendices to the extended version of this paper \citep{liu_motor_2021}.

\section{Preliminaries}

We now formally describe CCE \citep{moulin1978strategically} and JPSRO \citep{marris2021jpsroicml} before introducing our method \neupljpsro.

\subsection{Coarse Correlated Equilibrium (CCE)}
\label{sec:cce}

Normal-form (NF) games are the simplest game formulation where each player $p$ plays one of its actions $a_p \in \{ a_p^0, a_p^1, \dots\} = \cA_p$ simultaneously and receives a payoff $G_p: \cA \rightarrow \bR$ as a function of the joint action $a = (a_1, \dots, a_n)$, with $a \in \otimes_p \cA_p = \cA$ and $n$ the number of players. For player $p$, we denote $a = (a_p, a_\notp)$ with $a_\notp = (\dots, a_{p-1}, a_{p+1}, \dots)$, the actions played by all players except $p$. 
Let $\sigma(a) = \sigma(a_p, a_{\notp})$ denote the probability of players playing the joint action $a$ and $\sigma$ a probability distribution over the space of joint actions. A pure strategy is an action distribution that is deterministic and a mixed strategy is one that can be stochastic. The value for player $p$ under a mixed joint strategy $\sigma$ is defined as $\expt_{a \sim \sigma}[G_p(a)] = \expt_{a \sim \sigma}[G_p(a_p, a_\notp)]$. Let $\sigma_\notp(a_\notp) = \sum_{a_p} \sigma(a_p,a_\notp)$ be a marginal distribution over all players' action choices other than that of player $p$. Let $\Delta_\notp$ be the probability simplex of all such distributions. We similarly define marginals $\sigma(a_p) = \sum_{\notp} \sigma(a_p, a_\notp)$.

Let $\lfloor x \rfloor_{+} = \max(x, 0)$. The maximum incentive for $p$ to deviate from $\sigma$ by choosing action $a'_p$ is:
\begin{equation} \label{eq:delta_p}
    \delta_p(\sigma) = \bigl \lfloor \max_{a'_p \in \cA_p} \big( \expt_{a \sim \sigma}[G_p(a'_p, a_\notp) - G_p(a)] \big) \bigr \rfloor_{+}.
\end{equation}
$\sigma$ is an $\epsilon$-CCE if $\forall p, \delta_p(\sigma) \le \epsilon$.
A common metric for quantifying the approximation of $\sigma$ to a CCE is the CCE gap: $\delta(\sigma) = \sum_p \delta_p(\sigma) \ge 0$. A CCE gap is 0 if and only if $\sigma$ is a CCE. The value to a player under a CCE is its CCE value. CCE generalises NE in that every NE is a CCE but only CCE that factorise into marginals $\sigma(a) = \prod_p \sigma(a_p)$ are also NEs. Consequently, CCE always exist in finite games \citep{nash1951}.

NF CCE can be applied to EF games and we describe a specific construction that allows us to do so. In an EF game, player $p$ follows a policy $\pi_p(\cdot | s)$ that maps a state $s$ to a distribution over actions. Analogous to players taking actions in an NF game, we can construct an NF metagame where playing an action $a_p \in \cA_p$ executes a policy $\pi_p \in \Pi_p$ in the EF game for player $p$. Such a metagame would be large as it contains all enumerated policies as actions. We refer to the joint policy of all players $\pi = (\pi_1, \dots, \pi_n)$ as the joint policy and $\pi_\notp = (\dots, \pi_{p-1}, \pi_{p+1}, \dots)$ as the co-player joint policy for player $p$. In this metagame, the definition of an NF CCE applies. The $\max_{a'_p \in \cA_p}$ operator, from Equation~\ref{eq:delta_p}, amounts to $\max_{\pi'_p \in \Pi_p}$ or finding a policy $\pi'_p$ that maximises player $p$'s expected payoff to a co-player mixed-strategy $\sigma_\notp$. The policy that does so is a best-response (BR) to the co-player mixed-strategy $\sigma_\notp$. We study such NF metagames for the rest of this paper.

\subsection{Joint Policy-Space Response Oracle (JPSRO)}

The action space of an NF metagame can be intractable to enumerate. Nevertheless, we can study {\em restricted} NF metagames, whose action spaces are subsets of those of the full game. 
JPSRO \citep{marris2021jpsroicml} implements such an algorithm that converges to an NF CCE of an EF game (Algorithm~\ref{alg:jpsro}). 
Starting from an initial {\em restricted} metagame with a set of starting policies for each player $\Pi^0 = \{ \pi^0_p \}^n_{p=1}$, JPSRO computes a BR policy to co-player joint policies $\Pi^{t-1}_\notp$, sampled according to the marginal CCE $\sigma^{t-1}_\notp$ and adds it to player $p$'s set of policies at iteration $t$. This amounts to adding a metagame action for each player, leading to an expanded metagame at the next iteration. The expected payoff (EP) $G^t$ is re-evaluated at each iteration for all joint policies and $\sigma^t$ is re-computed using a CCE meta-strategy solver (MSS). $\sigma^t$ is referred to as a CCE of the {\em restricted} game, or a CCE mixed-strategy at iteration $t$.
This process terminates when a BR operator cannot find policies $\pi^t_p, \forall p$ that yield an $\epsilon$ improvement in expected payoff with $\max_p \delta^t_p = \bigl \lfloor \expt_{\pi \sim \sigma^{t-1}} [G_p(\pi^t_p, \pi_\notp) - G_p(\pi)] \bigr \rfloor_{+} < \epsilon$, resulting in an NF $\epsilon$-CCE of the full EF game. 

In short, instead of enumerating policies which is intractable for most NF metagames, JPSRO uses a BR operator to implement the maximisation step from Equation~\ref{eq:delta_p} and constructs a sequence of {\em restricted} metagames whose CCE converge to a CCE of the full game. The BR operator can be implemented using analytical solvers if available (as in \cite{marris2021jpsroicml}) or approximate methods such as RL. We refer to the former case as {\em exact} \jpsro and the latter {\em approximate}.

\begin{algorithm}
\centering
\caption{\jpsro(CCE) \citep{marris2021jpsroicml}}\label{alg:jpsro}
\begin{algorithmic}[1]
\State $\Pi^0_1, \dots, \Pi^0_n \eqdef \{\pi^0_1\}, \dots, \{\pi^0_n\}$
\State $G^0 \gets \textsc{EP}(\Pi^0)$
\State $\sigma^0 \gets \textsc{MSS}(G^0)$
\For{$t \in [1, \dots]$}
    \For{$p \in [1, \dots, n]$}
        \State $\pi^t_p, \delta^t_p \gets \textsc{BR}(\Pi^{t-1}_\notp, \sigma^{t-1}_\notp)$
        \State $\Pi^t_p \gets \Pi^{t-1}_p \cup \{\pi^t_p\}$
    \EndFor
    \If{$\max_p \delta^t_p < \epsilon$}  \Comment{$\epsilon$-CCE.}
        \State \Return $(\Pi^{t-1}, \sigma^{t-1})$
    \EndIf
    \State $G^t \gets \textsc{EP}(\Pi^t)$ \Comment{payoffs.}
    \State $\sigma^t \gets \textsc{MSS}(G^t)$ \Comment{CCE solver.}
\EndFor
\end{algorithmic}
\end{algorithm}

\section{\neupljpsro}

We now describe \neupljpsro (Algorithm~\ref{alg:neupljpsro}), an algorithm that builds on \jpsro but scales up to complex domains using function approximation and deep RL. The key idea behind \neupljpsro is to parameterise each policy for each player with a strategy embedding vector $\nu_p^t \in \mathbb{R}^{d}$, resulting in player-specific strategy embedding vectors $\cV = \{ \cV_1, \dots, \cV_n \}$ where $\cV_p = \{ \nu^0_p, \nu^1_p, \dots \}$ represents strategies available to player $p$. Each strategy embedding vector parameterises a policy $\Pi_\theta(\cdot | s, \nu)$, using a conditional neural network with parameters $\theta$ shared over all players' strategies. Strategy embedding vectors $\nu^t_p$ are randomly initialised and optimised jointly with the rest of the network parameters $\theta$. For conciseness, $\Pi^\cV_\theta$ denotes all sets of policies for all players. We omit subscript or superscript on strategy embeddings when referring to all players or all strategies of a player. Strategy embeddings are optimised end-to-end in the same way as the network parameters $\theta$. We use $\pi'(\cdot | s) \gets \pi(\cdot | s)$ to refer to policy $\pi'$ updating its action distribution to that of $\pi$ in state $s$. This can be implemented exactly in the tabular case, and approximately by minimising KL-divergence $\kld{\pi(\cdot | s)}{\pi'(\cdot | s)}$ when using function approximation. 

While \neupljpsro closely follows \jpsro, a key issue arises with this integrated approach to policy representation --- changes to one policy now affect behaviours of others. This could be beneficial in the form of skill transfer, as shown in \cite{liu2022neupl} (and the rest of this work), but may also affect the convergence guarantees of the algorithm. We discuss the convergence properties of population learning algorithms with shared representation in Section~\ref{sec:convergence} before explaining how \neupljpsro in theory guarantees convergence to a CCE with a continual learning approach. Section~\ref{sec:scaling} examines how function approximation can be used in each of the key operations of \neupljpsro to scale up to large games. In particular, best-response learning can benefit from skill transfer and expected payoff (EP) evaluation can build on the learned strategy embeddings. Our approach here generalises \neupl which is specialised to strategy representation in symmetric zero-sum games. The scaling approaches described here are also in contrast to exact \jpsro that does not generalise to large games, or approximate \jpsro, which quickly becomes computationally prohibitive at scale.

\subsection{Convergence to Equilibria}
\label{sec:convergence}

\begin{figure}
    \centering
    \includegraphics[width=0.9\columnwidth]{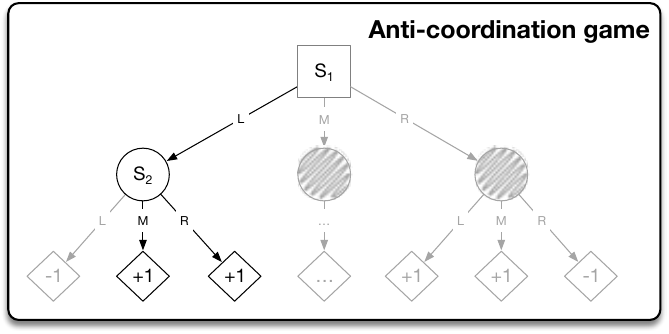}
    \caption{A turn-based two-player zero-sum game where player 1 publicly chooses a direction that player 2 is rewarded for avoiding. Terminal nodes show the payoffs of player 2.}
    \label{fig:unknown_unknown}
\end{figure}

A key element of the convergence arguments of \neupl \citep{liu2022neupl, liu2022simplex} relies on the stationarity of co-player policies at each best-responding iteration $\tau$. Given stationary co-player policies, it follows that a max-entropy RL algorithm would converge uniquely and the resulting policy would stand as a stationary co-player policy for all iterations $t > \tau$ that follow.
We show that the use of a max-entropy RL algorithm does {\em not} uniquely identify a best-response policy with a counter-example (Figure~\ref{fig:unknown_unknown}), unless additional assumptions are made on the dynamics of the game and the co-players. Specifically, we argue that max-entropy RL only results in stationary action distributions in {\em states it visits when best-responding}, which is a subset of all states where it may need to act as a co-player. This means that during best-response learning at iteration $t$, co-player policies developed at earlier iterations $\tau < t$ may be non-stationary in some states, and their best-responses may in turn change. In other words, best-response learning for iterations $t > \tau$ may be optimising towards ``moving targets'', and may never converge.

To make this concrete, consider the toy example shown in Figure~\ref{fig:unknown_unknown}. In this game, player 1 publicly declares a direction and player 2 is rewarded for avoiding it. Suppose player 1 plays a deterministic policy that always chooses ``L'' in state $S_1$ to which player 2 best-responds, a sample-based maximum-entropy BR operator would assign equal probability to ``M'' and ``R'' in state $S_2$, converging to the maximum-entropy BR which is stationary in all states reachable under the deterministic player 1 policy. Nevertheless, player 2 may be forced to act in additional (dashed) states in subsequent iterations with a different player 1 strategy -- indeed, with function approximation and shared representation, player 2 would behave unpredictably in these states yet any behaviour in these states would still constitute a valid best-response to player 1's original strategy. This is problematic for convergence, as the best-response to player 2's policy may change arbitrarily and there is no guarantee that the population would expand over all finite policies in the limit. This is true unless the best-responding policy can reach all states under the co-player joint policies (Definition~\ref{def:full_support}). Under the full-support condition, \neupl policies will remain stationary in all states. Indeed, the two domains considered in \cite{liu2022neupl} likely satisfy this full-support condition: \rps is an NF game, and in \rws co-player policies are stochastic, under partial observability and entropy-maximising, resulting in full-support reach probabilities for all states. This need not be the case in general games. In fact, the problem is particularly salient in games where best-response strategies are often deterministic (e.g. {\em goofspiel}, also known as the game of pure-strategies \citep{ross1971goofspiel}). We evaluate \neupljpsro in such domains in our experiments in Section~\ref{sec:results} to show its robustness in converging to a CCE even in the absence of the full-support condition.

\begin{definition}[Full-Support]
\label{def:full_support}
All states can be reached with positive probability at each iteration under the co-player joint-policy.
\end{definition}

How does \neupljpsro ensure convergence to an equilibria while maintaining the computational efficiency of shared strategy representation? The key idea behind \neupljpsro is to take a continual learning approach that removes the need for the full-support condition entirely.

\begin{algorithm}
\centering
\caption{\neupljpsro (Ours)}\label{alg:neupljpsro}
\begin{algorithmic}[1]
\State $\cV_1, \dots, \cV_n = \{\nu^0_p\}, \dots, \{\nu^0_n\}$  \Comment{With $\Pi_\theta(\cdot | s, \nu)$.}
\State $G^0 \gets \textsc{EP}(\Pi^\cV_\theta)$ \State $\sigma^0 \gets \textsc{MSS}(G^0)$
\For{$t \in [1, \dots]$} \Comment{\neupljpsro iterations.}
    \State $\hat\theta \gets \theta, \hat\cV \gets \cV$ \Comment{Reference policy parameters.}
    \For{$p \in [1, \dots, n]$}
        \State $\pi^t_p, \delta^t_p \gets \textsc{BR}(\Pi_{\hat \theta}^{\hat\cV}, \sigma^{t-1}_\notp)$ \Comment{See Section~\ref{sec:scaling_br}.}
        \State $\forall s, \Pi_\theta( \cdot | s, \nu^t_p) \gets \pi^t_p(\cdot | s)$ \Comment{Distill.}
        \State $\forall s, \Pi_\theta(\cdot | s, \nu_p) \gets \Pi_{\hat\theta}(\cdot | s, \hat\nu_p)$ \Comment{Regularise.}
        \State $\cV_p \gets \cV_p \cup \{ \nu^t_p \}$
    \EndFor
    \If{$\max_p \delta^t_p < \epsilon$}
        \State \Return $(\Pi_{\hat \theta}^{\hat\cV}, \sigma^{t-1})$
    \EndIf
    \State $G^t \gets \textsc{EP}(\Pi^\cV_\theta)$ \Comment{See Section~\ref{sec:scaling_ep}.}
    \State $\sigma^t \gets \textsc{MSS}(G^t)$
\EndFor
\end{algorithmic}
\end{algorithm}

Instead of concurrently optimising at all iterations as in \neupl, learning in \neupljpsro proceeds iteratively --- at iteration $t$, the best-response policies for each player $\pi^t_p, \forall p \in [n]$ are optimised against stationary co-player policies (we describe how $\pi^t_p$ are computed efficiently next). To ensure co-player stationarity, we introduce a set of reference policies $\Pi_{\hat\theta}(\cdot | s, \hat\nu^{\tau}_p), \forall \tau < t, \forall p \in [n]$, whose parameters $\hat\theta$ and $\hat\cV$ are held fixed for the duration of one iteration. Within one iteration, the best-response policies $\pi^t_p, \forall p \in [n]$ being optimised are distilled into the neural population $\forall s, \Pi_\theta( \cdot | s, \nu^t_p) \gets \pi^t_p(\cdot | s)$. At the same time, the behaviours of all existing strategies are held stationary via regularisation $\forall s, \Pi_\theta(\cdot | s, \nu_p) \gets \Pi_{\hat\theta}(\cdot | s, \hat\nu_p)$ in all states that player $p$ may reach. We recall that both distillation and regularisation are implemented as a minimisation problem of the KL-divergence between two policies. Our approach here draws inspiration from \cite{schwarz2018progress} where a conditional model continuously compresses existing skills while incorporating new ones with a shared skill latent space. We can now formally state the convergence guarantees of \neupljpsro with a proof that trivially extends from the convergence arguments of \jpsro \citep{marris2021jpsroicml}.

\begin{theorem}[CCE Convergence] \label{thm:neupljpsro}
When using a CCE meta-strategy solver in \neupljpsro, and when distill and regularise operators are exact, the sequence of mixed-strategy converges to a normal-form CCE under the meta-strategy distribution.
\end{theorem}

\begin{proof}
In this case, \neupljpsro is equivalent to \jpsro which is known to converge, as proved in \cite{marris2021jpsroicml}.
\end{proof}

Finally, we generalise the convergence arguments of \jpsro \citep{marris2021jpsroicml} which assumes that the BR operator returns deterministic policies, to also consider the case of stochastic policies as is often the case in policy-gradient RL methods. In fact, we show that by encouraging specific types of stochastic policies such as the ones that are entropy-maximising, we can guarantee that the population learning process terminates. We defer formal arguments to Appendix~\ref{app:cce_convergence} for completeness and use max-entropy RL algorithms for best-response solving in this work in practice.

\subsection{Scaling to large games}
\label{sec:scaling}

\begin{figure}
    \centering
    \includegraphics[width=0.8\columnwidth]{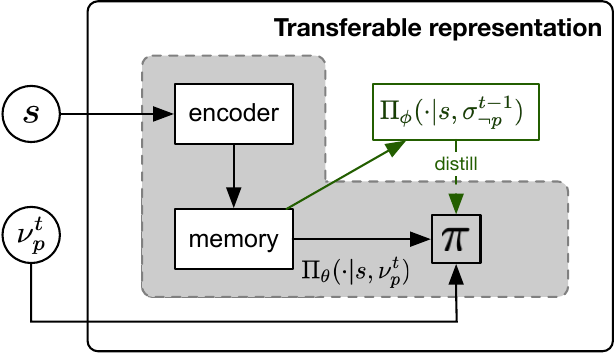}
    \caption{Efficient best-response solving by reusing transferable representation from the policy population $\Pi^\cV_\theta$. At iteration $t$, the policy head $\Pi_\phi$ (green) reuses the encoder and memory representation from $\Pi_\theta$ (gray) to learn a best-response to co-player mixed-strategy $\sigma^{t-1}_\notp$. The best-response policy is concurrently distilled into the neural population of policies $\Pi^\cV_\theta(\cdot | s, \nu^t_p)$ under the strategy embedding vector $\nu^t_p$.}
    \label{fig:scaling_br}
\end{figure}

\begin{figure*}
  \includegraphics[width=\textwidth]{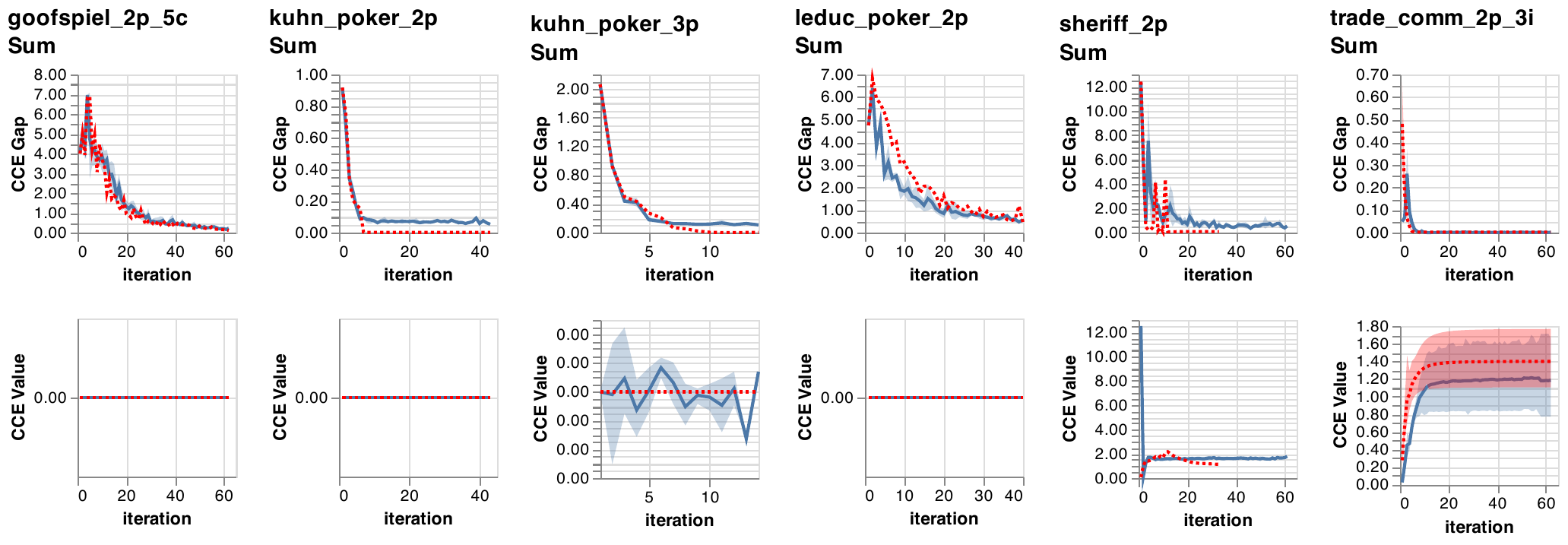}
  \caption{Exact CCE gaps and CCE values in 6 OpenSpiel games for \neupljpsro\xspace{\bf (Blue)} compared JPSRO {\bf (Red)} using exact best-response and expected payoff solvers averaged over 5 seeds.\label{fig:openspiel}}
\end{figure*}

\subsubsection{BR learning to CCE co-player mixed-strategies}
\label{sec:scaling_br}
At iteration $t$, \neupljpsro solves for a BR policy against co-player CCE mixed-strategy $\sigma^{t-1}_\notp$ for every player $p$. For games that cannot be solved analytically, approximate methods such as independent deep RL can be used. Nevertheless, doing so at every iteration would be impractical (as we show in Section~\ref{sec:ctf}). Learning from scratch when facing skilled co-players becomes challenging at later iterations, especially when basic skills (e.g. locomotion, perception and recurrent memory) are required for any meaningful exploration to occur.
Instead, we implement the BR operator using a policy network that benefits from transferable skills shared by the population of policies $\Pi^\cV_\theta$. Figure~\ref{fig:scaling_br} shows how an auxiliary policy head $\Pi_\phi(\cdot |s, \sigma^{t-1}_\notp)$ reuses the strategy-agnostic encoder and memory networks, and learns to best-respond to co-player mixed-strategies $\sigma^{t-1}_\notp$. Care needs to be taken when representing $\sigma^{t-1}_\notp$ in neural networks: the number of co-player joint strategies grows exponentially and a useful representation would reflect the player's probabilistic prior over co-players' joint strategies. We represent a co-player mixed-strategy $\sigma_\notp$ with a weighted representation $g(\cV, \sigma_\notp)$ as follows
\begin{equation}
    g(\cV, \sigma_\notp) = \sum_{a_\notp} \sigma_\notp(a_\notp) f( \dots,  \nu^{a_{p-1}}_{p-1}, \nu^{a_{p+1}}_{p+1}, \dots )
    \label{eq:weighted_coplayers}
\end{equation}
considering only {\sc top-k} co-player joint strategies $a_\notp$ under $\sigma_\notp$ for the summation operation. We describe further details in Appendix~\ref{app:representing_br} where we show that $k=96$ almost always leads to a lossless representation of co-player joint strategies in our experiments, as well as how player symmetry can be leveraged in the encoding function $f$ for further representation efficiency.

\subsubsection{BR learning to any co-player mixed-strategy}
The weighted co-player representation (Equation~\ref{eq:weighted_coplayers}) allows us to replicate results from \cite{liu2022simplex} too, where a conditional policy $\pi(\cdot | s, \sigma_\notp)$ responds Bayes-optimally to any co-player mixed-strategies $\sigma_\notp \in \Delta_\notp$. We demonstrate this potential for online adaption in Section~\ref{sec:cheetah_adapt} where a player collaborates with different co-players under uncertain prior beliefs. We make a simple modification to Algorithm~\ref{alg:neupljpsro} where we additionally sample arbitrary distributions over co-player strategies $\sigma_\notp \sim \Pr(\Delta_\notp)$ and optimise $\Pi_\phi(\cdot | s, \sigma_\notp)$ to best-respond accordingly. At convergence, $\Pi_\phi(\cdot | s, \sigma_\notp)$ behaves Bayes-optimally under any prior $\sigma_\notp \in \Delta_\notp$ over co-player joint-strategies.

\subsubsection{Expected payoff evaluation}
\label{sec:scaling_ep}
Payoff tensor $G^t$ needs to be evaluated at each iteration to update the metagame CCE mixed-strategy $\sigma^t \gets \textsc{MSS}(G^t)$. This can be costly: the number of joint strategies to evaluate grows exponentially in the number of players and estimating payoffs under each joint strategy may require many simulations in the absence of an analytical solver. 
We leverage learned strategy embedding vectors $\cV$ and continuously optimises a payoff estimator network $G(a) \gets \psi_w(\nu^{a_1}_1, \dots, \nu^{a_n}_n) \in \bR^n$ that predicts payoffs to each player under a joint strategy $a = (a_1, \dots, a_n)$. This network is used in all our experiments in lieu of the \textsc{EP} operator. Similar to \cite{liu2022neupl}, the payoff estimator network is trained to minimise the same loss function as the action-value function of the underlying RL agent. As $\psi_w$ is only conditioned on joint strategy embedding vectors and unaware of the state-action pairs, it is therefore regressing towards the expected returns for each player under a specific joint-strategy, with the expectation over the state and action distribution when players play out the specific joint-strategy $a$.  We describe how this payoff estimator network is implemented and optimised in Appendix~\ref{app:metagame_solving}. This payoff estimation network removes the need for evaluating payoff tensors through simulation, making payoff estimation at every \neupljpsro iteration practical and efficient.

\section{Results} \label{sec:results}

\begin{figure*}
    \centering
    \includegraphics[width=\textwidth]{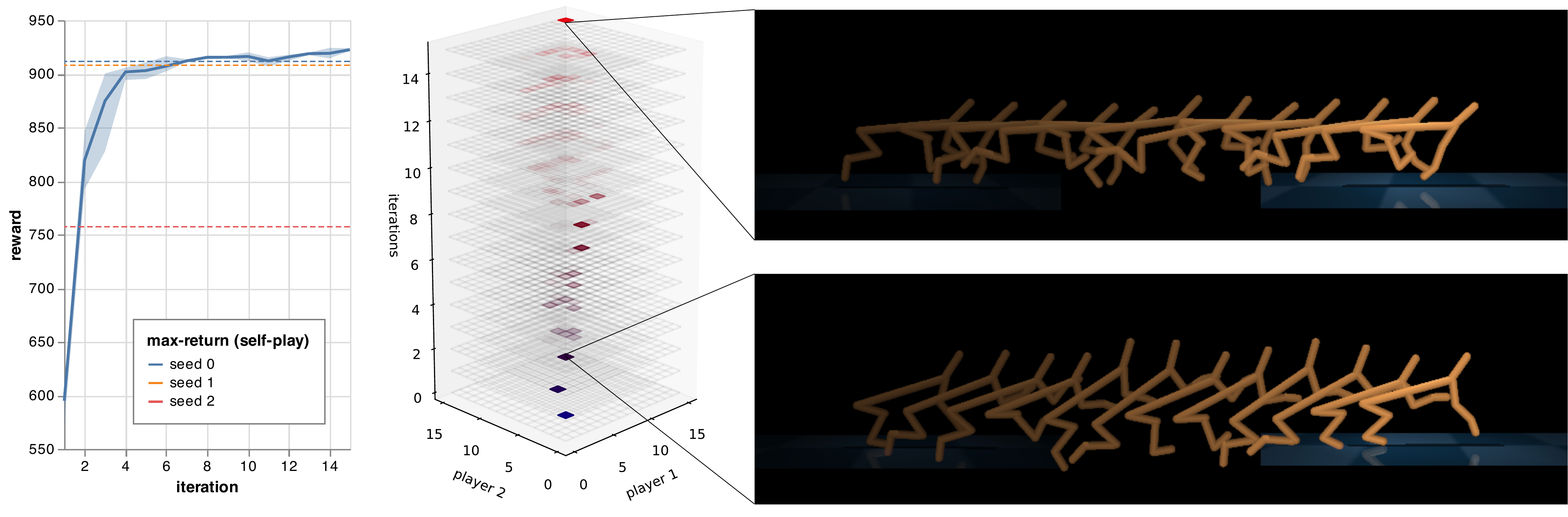}
    \caption{Emergence of cooperation in MuJoCo multiagent {\tt cheetah\_run}. {\bf (Left)} Expected returns achieved at each iteration (solid) compared to the maximum return obtained by independent trials where players optimise through self-play (dashed). The average return at iteration 16 is comparable to that of SoTA single-agent RL \citep{shahriari2022revisiting} {\bf (Middle)} The sequence of CCE best-responded to at each iteration for player 1 (rear leg) and player 2 (front leg). {\bf (Right)} Visualization of the learned behaviours at iteration 3 where the rear leg player raises the front leg player and at iteration 16 where both players cooperate competently.}
    \label{fig:cheetah}
\end{figure*}

We now present our empirical results that aim to answer two questions. First, we verify rigorously that \neupljpsro converges to a CCE in non-monotone strategy games for which exact CCE gaps and CCE values can be computed using analytical solvers. Second, we demonstrate its efficiency and generality by applying it to complex domains that involve realistic physics, partial observability and team-play. 
Additional technical details on compute, network architecture and game settings are provided in Appendix~\ref{app:results}.

\subsection{Convergence in n-player general-sum games}

\label{sec:openspiel}

To demonstrate empirical convergence to a CCE in n-player general-sum games, we investigated a diverse set of 6 strategy games from the OpenSpiel task suite \citep{LanctotEtAl2019OpenSpiel} as described in Appendix~\ref{app:full_openspiel}. Figure~\ref{fig:openspiel} summarises our results in each game reporting the sum of CCE gaps and CCE values across players when playing their equilibrium strategies of the {\em restricted} metagame at each iteration. 
For comparison, we show the same metrics for exact JPSRO in red, where entropy-maximising BRs are computed analytically at each iteration. 5 independent trials have been run for each game for each algorithms. For all games except {\tt trade\_comm}, the initial policy acts uniformly in all states\footnote{In {\tt trade\_comm}, learning a maximum-entropy BR to a uniform policy makes no progress: for a policy that ignores the message received, the BR would be agnostic to what message to send.}. The evaluation of CCE gaps and CCE values are exact in that the only input to the evaluation procedure from \neupljpsro are the sets of trained policies $\Pi^\cV_\theta$. The value of each policy, the optimal deviation actions as well as the CCE distributions are computed using analytical solvers.
For completeness, detailed results from each trial and for each player are shown in Appendix~\ref{app:full_openspiel}. For instance, inspecting per-player CCE values shows that \neupljpsro has recovered the last-mover advantage in poker. The sustained stability of the CCE values and CCE gaps at each \neupljpsro iteration suggest empirical co-player stationarity, necessary under our convergence arguments.

We make the following remarks. First, we observe empirical convergence towards a CCE in all games with both methods. In some games, \neupljpsro discovered and represented up to 64 policies for each player, far exceeding the size of the population reported in prior works \citep{liu2022neupl}. This demonstrates the potential of \neupljpsro to converge in games with long strategy cycles such as {\tt goofspiel}. 
Second, neither \neupljpsro nor JPSRO converged to CCE with specific properties. In particular, the values of the CCE do not converge to the values of the maximum-welfare CCE in non-zero-sum games ({\tt sheriff} and {\tt trade\_comm})\footnote{Our results for the JPSRO baseline differ from what have been reported in \cite{marris2021jpsroicml}. This is because the default analytical BR solver in OpenSpiel chooses the first action deterministically among indifferent ones. We used a maximum-entropy solver instead, which removes this implicit coordination bias.}. Equilibrium selection \citep{barman2015finding} remains an open question: every CCE describes a rational, stable state of the system in that no one has an incentive to deviate from their equilibrium strategy but only certain equilibria are socially valuable or fair. \neupljpsro converges to one such equilibrium, but to which remains unclear. 

Additionally, we note that in some games, \neupljpsro, an approximate method, has observed faster convergence to a CCE in some games than exact \jpsro in early iterations (e.g. Leduc Poker in the first 20 iterations). This maybe counter-intuitive at first, but we shall explain why this could be the case in practice through an example. Consider the game of \rps and suppose we start off the population with a deterministic always-rock policy. An exact best-response to this initial policy would be always-paper and an approximate one (perhaps due to entropy maximisation) may randomise but with a strong preference for paper (e.g. 80\%-paper, 10\% rock and 10\% scissors). In each scenario, the CCE between the first two strategies amounts to always playing the latest strategy --- always-paper is played in the former, and a mostly-paper strategy is played in the latter. It is therefore unsurprising that the exact case has a higher CCE gap as always-paper is more exploitable (to always-scissors) than its approximate, randomising counterpart. Indeed, in population learning algorithms such as \psro, the rate of convergence is not necessarily faster when the best-response solver is exact. The empirical rate of convergence to an equilibrium depends on the dynamics of the game and the choice of initial policy as well. Nevertheless, we show \neupljpsro to converge empirically, at a rate comparable to exact \jpsro.

\begin{figure*}
    \centering
    \includegraphics[width=\textwidth]{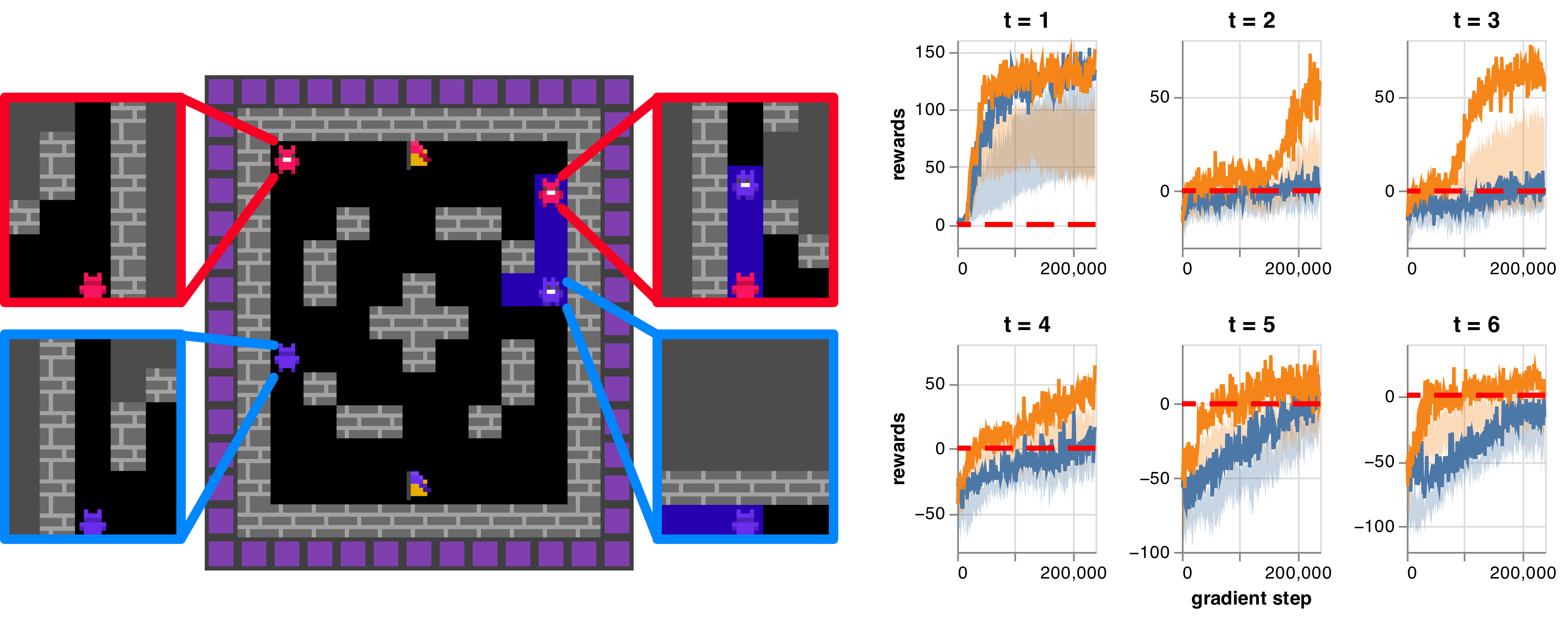}
    \caption{{\bf (Left)} A 4-player {\tt capture-the-flag} environment showing first-person views for each player. {\bf (Right)} Convergence to a CCE shown by the diminishing incentive to deviate to an independent BR across iterations. (Blue) Expected returns of independent RL exploiter policies optimised against the marginal CCE mixed-strategy $\sigma^t_\notp$ from \neupljpsro at iteration $t$. (Orange) Same as Blue, but initialised with pre-trained encoder and memory network parameters (as in Figure~\ref{fig:scaling_br}). (Red) CCE values at each \neupljpsro iteration. Solid lines show the optimistic {\em best-of-six} exploiter returns.}
    \label{fig:ctf}
\end{figure*}

\subsection{Online adaptation in multiagent MuJoCo domains}

\label{sec:cheetah_adapt}

Humans often adapt to each other in an online fashion, whether it's in competition, cooperation, or a mixture of both. Through iterative BR solving, population learning can generate diverse and strategically relevant behaviours. Prior work leveraged this property in a competitive setting \citep{liu2022simplex}, where a single conditional policy $\pi(\cdot | s, \sigma_\notp)$ is trained to best-respond to arbitrary metagame mixed-strategies $\sigma_\notp \in \Delta_\notp$. By interpreting $\sigma_\notp$ as a {\em prior} over co-player joint strategies, $\pi(\cdot | s, \sigma_\notp)$ trades off exploration and exploitation Bayes-optimally, maximising its expected returns. We study if \neupljpsro can lead to similar adaptive behaviours but in a cooperative setting. We construct a cooperative common-payoff physically-simulated game using {\tt cheetah\_run} \citep{tassa2020dm_control} where player 1 and player 2 control the rear and front legs respectively and receive the same forward velocity reward. Both players observe all joint positions of the cheetah embodiment but do not observe each other's joint actuation. To do well, players must build a model of the co-player at play and react accordingly.

Our goal is to verify that by letting players control parts of the same body, the population of policies $\Pi^\cV_\theta$ would develop diverse behaviour patterns, while the conditional BR policy $\Pi_\phi(\cdot | s, \sigma_\notp)$ learns to adapt to different partners in an online fashion.
Figure~\ref{fig:cheetah} (Left) shows our findings.  With \neupljpsro, the two players discovered a sequence of 16 strategies using the same conditional network, continuously improving the expected returns across iterations. The final averaged return outperforms self-play baselines with less variances across seeds. Figure~\ref{fig:cheetah} (Middle) visualised the sequence of CCE joint distributions discovered over 16 iterations. In earlier iterations, the player controlling the rear leg learned to raise its partner so as to minimise the disruption caused by an unskilled co-player (shown in Figure~\ref{fig:cheetah} (Bottom-Right)). This strategy became dominant in earlier iterations as it maximises returns for the pair. In subsequent iterations, the front leg player discovered policies that can cooperate with the rear leg player, leading to coordinated joint policies as shown in Figure~\ref{fig:cheetah} (Top-Right). We show this emergence of partnership visually over the first 4 iterations\footnote{cheetah-run: \url{https://youtu.be/-bBR6Vtu0sI}}.
More interestingly, by conditioning the final BR policy $\Pi_\phi(\cdot | s, \bar{\sigma}_\notp)$ with a uniform co-player prior $\bar{\sigma}_\notp$, we observed that the same policy is capable of probing its front leg partner through interaction and adapt accordingly online\footnote{cheetah-adapt: \url{https://youtu.be/zMwhWgafAK4}.}. Following initial feedback, the rear leg player either took control of the full body if its partner appears uncooperative or worked effectively with a front leg partner that proved competent. 
Our findings are consistent with \cite{liu2022simplex}, who considered agents adapting to multiple opponents in a competitive setting. We show here that similar adaptive behaviors can emerge in a cooperative setting with asymmetric roles using \neupljpsro.

\subsection{Strategic team-play in capture-the-flag}
\label{sec:ctf}

\neupljpsro efficiently scales up to games requiring transferable skills while converging to an approximate CCE. In this section, we investigate {\tt capture-the-flag}, a 4-player game where players compete in teams to capture the opponents' flag and bring it back to the home base under partial observability. 

A global view of the environment as well as players' first-person observation of the game are shown in Figure~\ref{fig:ctf} (Left). Compared to the standard environment described in \cite{leibo2021meltingpot}, we modified the visibility mechanism such that brick walls would restrict players' view (shown in grey). This puts further emphasis on players' abilities to infer others' strategies under partial visibility. 
We used the default rewards, which are sparse: players on the same team receive a reward of +10 (-10) upon capturing (conceding) a flag and zero otherwise. Players on the same team do not have explicit communication channels: they can observe floor paint left by co-players, current status of flag capture and the behaviours of other players to form implicit conventions. The opponent's flag is captured when returned to one's home if and only if their own flag remains at home too. Players may paint the floor and they cannot move if the tile they stand on has been painted over by an opponent. A player is removed if they are tagged twice within a short time and they recover more quickly if standing on tiles painted in their own colour.
Despite the complexity of the game, \neupljpsro successfully discovered a sequence of 8 strategies, using a single conditional network. 
It is challenging to offer a succinct view of the learning dynamics considering the size of the payoff tensor, but Appendix~\ref{app:ctf_results} shows an overall trend of progress: joint strategies developed at later iterations tend to be less exploitable, dominating earlier combinations of opponent strategies. 

We visualise these strategies\footnote{capture-the-flag: \url{https://youtu.be/z5EeMfcOo7A}. Each player's first-person view is annotated in white with the iteration of the strategy at play. The cumulative return to each team is shown at the top.}. While this environment is a simplified variant of the classic 3D game studied in \cite{jaderberg2019human}, we observe similar behaviour patterns to emerge here as well. Throughout the first few iterations, players incrementally learned to implement coordinated strategies such as ``Home Base Defence'', ``Opponent Base Camping'' and finally ``Teammate Following''. Other skills such as timing the tagging cool-down mechanism have emerged, too.

It is difficult to verify convergence to a CCE in a game of this complexity. Nevertheless, we offer empirical evidence of convergence to a CCE by independently training best-responding RL policies against marginal CCE mixed-strategies of the restricted metagames at each iteration $\{ \sigma^t \}^T_{t=0}$, using standard RL algorithms. This is similar to the exact CCE gap evaluation in Section~\ref{sec:openspiel} but with an approximate BR solver. Figure~\ref{fig:ctf} (Right) shows our findings. Compared to the CCE value to player $p$ when playing according to the joint CCE distribution $\sigma^t$ (Red dashed), a standard RL algorithm managed to find policies with better returns in the first 5 iterations (Orange and Blue). These iterations have not converged to a CCE of the game as profitable deviations can be made by unilaterally switching to these independently optimised BR policies. At iteration 6, independent RL training can no longer outperform. This shows that we are in close proximity of a CCE and little room for improvement remains compared to what player $p$'s equilibrium strategy already implements. We note that we report the {\em maximum} expected returns observed over 6 independently trained BR policies. This is because our goal is to certify if {\em any} profitable deviation actions can be found. 
We report the average case in Appendix~\ref{app:ctf_results} where independent RL struggled to outperform the CCE equilibrium strategies as soon as iteration 5. 
Our exploiter results suggest that RL policies optimised via self-play would {\em at best} match the CCE solution of \neupljpsro. We note that the CCE value to each player at every iteration must be zero, as the game is zero-sum. 

An orthogonal benefit of \neupljpsro lies in its potential in transfer learning across players and strategies. This is particularly attractive in domains such as ours that afford transferable skills. 
To do well, players must learn to represent visual observation history so as to infer other players' strategies under partial observability. In close combat, players must understand the cool-down mechanism of the tagging action and proactively retreat to a safe distance in a stand off. 
Figure~\ref{fig:ctf} (Right) confirms our hypothesis, showing that a randomly initialised RL policy struggled to best-respond to strong opponents (Blue) when a policy partly initialised with pre-trained encoder and memory network parameters succeeded (Orange). This highlights the importance of transfer learning in complex games --- strong opponents tend to create difficult exploration problems for randomly initialised RL policies. In this instance, approximate JPSRO would have led to the incorrect conclusion that no further improvement can be made after 3 iterations, forfeiting strategically interesting joint-strategies of this game. \neupljpsro naturally promotes a progressive learning curriculum and benefit from transfer learning of strategy-agnostic skills.

\section{Related Work}

A rich body of literature have focused on providing convergence guarantees to equilibria in games of different degrees of generalities.  Two-player zero-sum games, in particular, attracted attention given the tractability and interchangeability of NE in this setting. Families of NE-convergent methods have been developed, with recent successes in scaling to complex EF games of real-world interests \citep{brown1951, mcmahan_planning_2003, lanctot2017unified, perolat2021poincare}.
Progress beyond two-player zero-sum games trails behind in comparison. Principled no-regret learning methods \citep{cesa2006prediction, daskalakis2021near} have been proposed, converging to the generalised solution concept of (C)CE \citep{aumann1974subjectivity, moulin1978strategically}. Nevertheless, there has been limited success in deploying these methods to in many practical applications due to their computational challenges in scaling to large games.

Of particular relevance to our work are population learning methods that leverage latest advances in function approximation and deep RL. Several works proposed to learn a population of deep RL agents and observed emergent complex behaviours through population-based interactions \citep{bansal2018emergent, jaderberg2019human, liu2022motor, liu2019emergent, vinyals2019grandmaster}. While these works demonstrated what could be achieved by multiagent RL at scale, their convergence characteristics remain to be understood \citep{garnelo2021pick}. In contrast, \cite{lanctot2017unified, mcaleer_pipeline_2021, smith2020iterative, liu2022neupl} proposed methods that leverage deep RL while retaining the game-theoretic convergence guarantees. Fewer works in this category ventured beyond two-player zero-sum games. \cite{Muller2020A} built on the scalable meta-strategy solver of $\alpha$-rank \citep{omidshafiei2019alpha}. \cite{marris2021jpsroicml} extended PSRO to the n-player general-sum case with theoretical convergence guarantees to CCE but does not scale up to domains with complex action and observation spaces.

\section{Limitations}

A common limitation of equilibrium convergent population learning algorithms is the lack of guarantee on which equilibrium the population would converge to. This applies to our work too (Section~\ref{sec:openspiel}). While we can expect convergence to an equilibrium, we cannot predict if the solution would be {\em desirable} (e.g. that it is social-welfare maximising). In competitive games, equilibrium selection is less critical as NEs are interchangeable. This is in contrast to general-sum games where players might find themselves significantly less well off in some equilibria than others. Computationally, while \neupljpsro addressed many of the bottlenecks of prior methods, it remains challenging to scale up to {\em many} players each with {\em many} strategies due to the size of the payoff tensor. A future direction would be to consider sample-based equilibria solutions, without needing to tabulate the entire payoff tensor upfront.

\section{Conclusions}

We proposed \neupljpsro as an efficient and scalable algorithm to solving n-player general-sum EF games that provably converges to a NF CCE. We verified its convergence empirically, across diverse test domains ranging from research strategy games to games that require deep RL. We showed that \neupljpsro can adapt to diverse co-players in non-zero-sum settings and demonstrated the importance of transfer learning in solving games with transferable skills. Our method is computationally accessible, paving the way for deploying game-theoretic solutions to real-world general games.

\bibliographystyle{ACM-Reference-Format} 
\bibliography{main}


\begin{thebibliography}{48}

%%% the Web address.

\ifx \showCODEN    \undefined \def \showCODEN     #1{\unskip}     \fi
\ifx \showDOI      \undefined \def \showDOI       #1{#1}\fi
\ifx \showISBNx    \undefined \def \showISBNx     #1{\unskip}     \fi
\ifx \showISBNxiii \undefined \def \showISBNxiii  #1{\unskip}     \fi
\ifx \showISSN     \undefined \def \showISSN      #1{\unskip}     \fi
\ifx \showLCCN     \undefined \def \showLCCN      #1{\unskip}     \fi
\ifx \shownote     \undefined \def \shownote      #1{#1}          \fi
\ifx \showarticletitle \undefined \def \showarticletitle #1{#1}   \fi
\ifx \showURL      \undefined \def \showURL       {\relax}        \fi
% invisible to TeX
\providecommand\bibfield[2]{#2}
\providecommand\bibinfo[2]{#2}
\providecommand\natexlab[1]{#1}
\providecommand\showeprint[2][]{arXiv:#2}

\bibitem[\protect\citeauthoryear{Abdolmaleki, Springenberg, Tassa, Munos,
  Heess, and Riedmiller}{Abdolmaleki et~al\mbox{.}}{2018}]%
        {abdolmaleki2018maximum}
\bibfield{author}{\bibinfo{person}{Abbas Abdolmaleki},
  \bibinfo{person}{Jost~Tobias Springenberg}, \bibinfo{person}{Yuval Tassa},
  \bibinfo{person}{Remi Munos}, \bibinfo{person}{Nicolas Heess}, {and}
  \bibinfo{person}{Martin Riedmiller}.} \bibinfo{year}{2018}\natexlab{}.
\newblock \showarticletitle{Maximum a Posteriori Policy Optimisation}. In
  \bibinfo{booktitle}{\emph{International Conference on Learning
  Representations}}.
\newblock
\urldef\tempurl%
\url{https://openreview.net/forum?id=S1ANxQW0b}
\showURL{%
\tempurl}


\bibitem[\protect\citeauthoryear{Aumann}{Aumann}{1974}]%
        {aumann1974subjectivity}
\bibfield{author}{\bibinfo{person}{Robert~J Aumann}.}
  \bibinfo{year}{1974}\natexlab{}.
\newblock \showarticletitle{Subjectivity and correlation in randomized
  strategies}.
\newblock \bibinfo{journal}{\emph{Journal of mathematical Economics}}
  \bibinfo{volume}{1}, \bibinfo{number}{1} (\bibinfo{year}{1974}),
  \bibinfo{pages}{67--96}.
\newblock


\bibitem[\protect\citeauthoryear{Bansal, Pachocki, Sidor, Sutskever, and
  Mordatch}{Bansal et~al\mbox{.}}{2018}]%
        {bansal2018emergent}
\bibfield{author}{\bibinfo{person}{Trapit Bansal}, \bibinfo{person}{Jakub
  Pachocki}, \bibinfo{person}{Szymon Sidor}, \bibinfo{person}{Ilya Sutskever},
  {and} \bibinfo{person}{Igor Mordatch}.} \bibinfo{year}{2018}\natexlab{}.
\newblock \showarticletitle{Emergent Complexity via Multi-Agent Competition}.
  In \bibinfo{booktitle}{\emph{International Conference on Learning
  Representations}}.
\newblock
\urldef\tempurl%
\url{https://openreview.net/forum?id=Sy0GnUxCb}
\showURL{%
\tempurl}


\bibitem[\protect\citeauthoryear{Barman and Ligett}{Barman and Ligett}{2015}]%
        {barman2015finding}
\bibfield{author}{\bibinfo{person}{Siddharth Barman} {and}
  \bibinfo{person}{Katrina Ligett}.} \bibinfo{year}{2015}\natexlab{}.
\newblock \showarticletitle{Finding any nontrivial coarse correlated
  equilibrium is hard}.
\newblock \bibinfo{journal}{\emph{ACM SIGecom Exchanges}} \bibinfo{volume}{14},
  \bibinfo{number}{1} (\bibinfo{year}{2015}), \bibinfo{pages}{76--79}.
\newblock


\bibitem[\protect\citeauthoryear{Brown}{Brown}{1951}]%
        {brown1951}
\bibfield{author}{\bibinfo{person}{George~W. Brown}.}
  \bibinfo{year}{1951}\natexlab{}.
\newblock \showarticletitle{Iterative solution of games by fictitious play.}
\newblock \bibinfo{journal}{\emph{Activity Analysis of Production and
  Allocation}} (\bibinfo{year}{1951}).
\newblock
\showeprint{13(1):374–376}~[cs.LG]


\bibitem[\protect\citeauthoryear{Brown and Sandholm}{Brown and
  Sandholm}{2018}]%
        {brown2018superhuman}
\bibfield{author}{\bibinfo{person}{Noam Brown} {and} \bibinfo{person}{Tuomas
  Sandholm}.} \bibinfo{year}{2018}\natexlab{}.
\newblock \showarticletitle{Superhuman AI for heads-up no-limit poker: Libratus
  beats top professionals}.
\newblock \bibinfo{journal}{\emph{Science}} \bibinfo{volume}{359},
  \bibinfo{number}{6374} (\bibinfo{year}{2018}), \bibinfo{pages}{418--424}.
\newblock


\bibitem[\protect\citeauthoryear{Campbell, Hoane~Jr, and Hsu}{Campbell
  et~al\mbox{.}}{2002}]%
        {campbell2002deep}
\bibfield{author}{\bibinfo{person}{Murray Campbell}, \bibinfo{person}{A~Joseph
  Hoane~Jr}, {and} \bibinfo{person}{Feng-hsiung Hsu}.}
  \bibinfo{year}{2002}\natexlab{}.
\newblock \showarticletitle{Deep blue}.
\newblock \bibinfo{journal}{\emph{Artificial intelligence}}
  \bibinfo{volume}{134}, \bibinfo{number}{1-2} (\bibinfo{year}{2002}),
  \bibinfo{pages}{57--83}.
\newblock


\bibitem[\protect\citeauthoryear{Cesa-Bianchi and Lugosi}{Cesa-Bianchi and
  Lugosi}{2006}]%
        {cesa2006prediction}
\bibfield{author}{\bibinfo{person}{Nicolo Cesa-Bianchi} {and}
  \bibinfo{person}{G{\'a}bor Lugosi}.} \bibinfo{year}{2006}\natexlab{}.
\newblock \bibinfo{booktitle}{\emph{Prediction, learning, and games}}.
\newblock \bibinfo{publisher}{Cambridge university press}.
\newblock


\bibitem[\protect\citeauthoryear{Daskalakis, Fishelson, and
  Golowich}{Daskalakis et~al\mbox{.}}{2021}]%
        {daskalakis2021near}
\bibfield{author}{\bibinfo{person}{Constantinos Daskalakis},
  \bibinfo{person}{Maxwell Fishelson}, {and} \bibinfo{person}{Noah Golowich}.}
  \bibinfo{year}{2021}\natexlab{}.
\newblock \showarticletitle{Near-optimal no-regret learning in general games}.
\newblock \bibinfo{journal}{\emph{Advances in Neural Information Processing
  Systems}}  \bibinfo{volume}{34} (\bibinfo{year}{2021}),
  \bibinfo{pages}{27604--27616}.
\newblock


\bibitem[\protect\citeauthoryear{Daskalakis, Goldberg, and
  Papadimitriou}{Daskalakis et~al\mbox{.}}{2009}]%
        {daskalakis2009complexity}
\bibfield{author}{\bibinfo{person}{Constantinos Daskalakis},
  \bibinfo{person}{Paul~W Goldberg}, {and} \bibinfo{person}{Christos~H
  Papadimitriou}.} \bibinfo{year}{2009}\natexlab{}.
\newblock \showarticletitle{The complexity of computing a Nash equilibrium}.
\newblock \bibinfo{journal}{\emph{SIAM J. Comput.}} \bibinfo{volume}{39},
  \bibinfo{number}{1} (\bibinfo{year}{2009}), \bibinfo{pages}{195--259}.
\newblock


\bibitem[\protect\citeauthoryear{Farina, Ling, Fang, and Sandholm}{Farina
  et~al\mbox{.}}{2019}]%
        {farina2019correlation}
\bibfield{author}{\bibinfo{person}{Gabriele Farina}, \bibinfo{person}{Chun~Kai
  Ling}, \bibinfo{person}{Fei Fang}, {and} \bibinfo{person}{Tuomas Sandholm}.}
  \bibinfo{year}{2019}\natexlab{}.
\newblock \showarticletitle{Correlation in extensive-form games: Saddle-point
  formulation and benchmarks}.
\newblock \bibinfo{journal}{\emph{Advances in Neural Information Processing
  Systems}}  \bibinfo{volume}{32} (\bibinfo{year}{2019}).
\newblock


\bibitem[\protect\citeauthoryear{Garnelo, Czarnecki, Liu, Tirumala, Oh, Gidel,
  van Hasselt, and Balduzzi}{Garnelo et~al\mbox{.}}{2021}]%
        {garnelo2021pick}
\bibfield{author}{\bibinfo{person}{Marta Garnelo},
  \bibinfo{person}{Wojciech~Marian Czarnecki}, \bibinfo{person}{Siqi Liu},
  \bibinfo{person}{Dhruva Tirumala}, \bibinfo{person}{Junhyuk Oh},
  \bibinfo{person}{Gauthier Gidel}, \bibinfo{person}{Hado van Hasselt}, {and}
  \bibinfo{person}{David Balduzzi}.} \bibinfo{year}{2021}\natexlab{}.
\newblock \showarticletitle{Pick Your Battles: Interaction Graphs as
  Population-Level Objectives for Strategic Diversity}. In
  \bibinfo{booktitle}{\emph{AAMAS}}.
\newblock


\bibitem[\protect\citeauthoryear{Jaderberg, Czarnecki, Dunning, Marris, Lever,
  Castaneda, Beattie, Rabinowitz, Morcos, Ruderman, et~al\mbox{.}}{Jaderberg
  et~al\mbox{.}}{2019}]%
        {jaderberg2019human}
\bibfield{author}{\bibinfo{person}{Max Jaderberg}, \bibinfo{person}{Wojciech~M
  Czarnecki}, \bibinfo{person}{Iain Dunning}, \bibinfo{person}{Luke Marris},
  \bibinfo{person}{Guy Lever}, \bibinfo{person}{Antonio~Garcia Castaneda},
  \bibinfo{person}{Charles Beattie}, \bibinfo{person}{Neil~C Rabinowitz},
  \bibinfo{person}{Ari~S Morcos}, \bibinfo{person}{Avraham Ruderman},
  {et~al\mbox{.}}} \bibinfo{year}{2019}\natexlab{}.
\newblock \showarticletitle{Human-level performance in 3D multiplayer games
  with population-based reinforcement learning}.
\newblock \bibinfo{journal}{\emph{Science}} \bibinfo{volume}{364},
  \bibinfo{number}{6443} (\bibinfo{year}{2019}), \bibinfo{pages}{859--865}.
\newblock


\bibitem[\protect\citeauthoryear{Kingma and Ba}{Kingma and Ba}{2015}]%
        {KingmaB14}
\bibfield{author}{\bibinfo{person}{Diederik~P. Kingma} {and}
  \bibinfo{person}{Jimmy Ba}.} \bibinfo{year}{2015}\natexlab{}.
\newblock \showarticletitle{Adam: {A} Method for Stochastic Optimization}. In
  \bibinfo{booktitle}{\emph{3rd International Conference on Learning
  Representations, {ICLR} 2015, San Diego, CA, USA, May 7-9, 2015, Conference
  Track Proceedings}}, \bibfield{editor}{\bibinfo{person}{Yoshua Bengio} {and}
  \bibinfo{person}{Yann LeCun}} (Eds.).
\newblock
\urldef\tempurl%
\url{http://arxiv.org/abs/1412.6980}
\showURL{%
\tempurl}


\bibitem[\protect\citeauthoryear{Kuhn}{Kuhn}{1950}]%
        {kuhn1950simplified}
\bibfield{author}{\bibinfo{person}{Harold~W Kuhn}.}
  \bibinfo{year}{1950}\natexlab{}.
\newblock \showarticletitle{A simplified two-person poker}.
\newblock \bibinfo{journal}{\emph{Contributions to the Theory of Games}}
  \bibinfo{volume}{1} (\bibinfo{year}{1950}), \bibinfo{pages}{97--103}.
\newblock


\bibitem[\protect\citeauthoryear{Kuhn and Tucker}{Kuhn and Tucker}{1957}]%
        {kuhn1957efg}
\bibfield{author}{\bibinfo{person}{H.~W. Kuhn} {and} \bibinfo{person}{AW
  Tucker}.} \bibinfo{year}{1957}\natexlab{}.
\newblock \showarticletitle{Extensive games and the problem and information}.
\newblock \bibinfo{journal}{\emph{Contributions to the Theory of Games, II,
  Annals of Mathematical Studies}}  \bibinfo{volume}{28}
  (\bibinfo{year}{1957}), \bibinfo{pages}{193–216}.
\newblock


\bibitem[\protect\citeauthoryear{Lanctot, Lockhart, Lespiau, Zambaldi,
  Upadhyay, P\'{e}rolat, Srinivasan, Timbers, Tuyls, Omidshafiei, Hennes,
  Morrill, Muller, Ewalds, Faulkner, Kram\'{a}r, Vylder, Saeta, Bradbury, Ding,
  Borgeaud, Lai, Schrittwieser, Anthony, Hughes, Danihelka, and
  Ryan-Davis}{Lanctot et~al\mbox{.}}{2019}]%
        {LanctotEtAl2019OpenSpiel}
\bibfield{author}{\bibinfo{person}{Marc Lanctot}, \bibinfo{person}{Edward
  Lockhart}, \bibinfo{person}{Jean-Baptiste Lespiau}, \bibinfo{person}{Vinicius
  Zambaldi}, \bibinfo{person}{Satyaki Upadhyay}, \bibinfo{person}{Julien
  P\'{e}rolat}, \bibinfo{person}{Sriram Srinivasan}, \bibinfo{person}{Finbarr
  Timbers}, \bibinfo{person}{Karl Tuyls}, \bibinfo{person}{Shayegan
  Omidshafiei}, \bibinfo{person}{Daniel Hennes}, \bibinfo{person}{Dustin
  Morrill}, \bibinfo{person}{Paul Muller}, \bibinfo{person}{Timo Ewalds},
  \bibinfo{person}{Ryan Faulkner}, \bibinfo{person}{J\'{a}nos Kram\'{a}r},
  \bibinfo{person}{Bart~De Vylder}, \bibinfo{person}{Brennan Saeta},
  \bibinfo{person}{James Bradbury}, \bibinfo{person}{David Ding},
  \bibinfo{person}{Sebastian Borgeaud}, \bibinfo{person}{Matthew Lai},
  \bibinfo{person}{Julian Schrittwieser}, \bibinfo{person}{Thomas Anthony},
  \bibinfo{person}{Edward Hughes}, \bibinfo{person}{Ivo Danihelka}, {and}
  \bibinfo{person}{Jonah Ryan-Davis}.} \bibinfo{year}{2019}\natexlab{}.
\newblock \showarticletitle{{OpenSpiel}: A Framework for Reinforcement Learning
  in Games}.
\newblock \bibinfo{journal}{\emph{CoRR}}  \bibinfo{volume}{abs/1908.09453}
  (\bibinfo{year}{2019}).
\newblock
\showeprint[arxiv]{1908.09453}~[cs.LG]
\urldef\tempurl%
\url{http://arxiv.org/abs/1908.09453}
\showURL{%
\tempurl}


\bibitem[\protect\citeauthoryear{Lanctot, Zambaldi, Gruslys, Lazaridou, Tuyls,
  Perolat, Silver, and Graepel}{Lanctot et~al\mbox{.}}{2017}]%
        {lanctot2017unified}
\bibfield{author}{\bibinfo{person}{Marc Lanctot}, \bibinfo{person}{Vinicius
  Zambaldi}, \bibinfo{person}{Audrunas Gruslys}, \bibinfo{person}{Angeliki
  Lazaridou}, \bibinfo{person}{Karl Tuyls}, \bibinfo{person}{Julien Perolat},
  \bibinfo{person}{David Silver}, {and} \bibinfo{person}{Thore Graepel}.}
  \bibinfo{year}{2017}\natexlab{}.
\newblock \showarticletitle{A Unified Game-Theoretic Approach to Multiagent
  Reinforcement Learning}. In \bibinfo{booktitle}{\emph{Advances in Neural
  Information Processing Systems}},
  \bibfield{editor}{\bibinfo{person}{I.~Guyon}, \bibinfo{person}{U.~V.
  Luxburg}, \bibinfo{person}{S.~Bengio}, \bibinfo{person}{H.~Wallach},
  \bibinfo{person}{R.~Fergus}, \bibinfo{person}{S.~Vishwanathan}, {and}
  \bibinfo{person}{R.~Garnett}} (Eds.), Vol.~\bibinfo{volume}{30}.
  \bibinfo{publisher}{Curran Associates, Inc.}
\newblock
\urldef\tempurl%
\url{https://proceedings.neurips.cc/paper/2017/file/3323fe11e9595c09af38fe67567a9394-Paper.pdf}
\showURL{%
\tempurl}


\bibitem[\protect\citeauthoryear{Leibo, Due{\~n}ez-Guzman, Vezhnevets, Agapiou,
  Sunehag, Koster, Matyas, Beattie, Mordatch, and Graepel}{Leibo
  et~al\mbox{.}}{2021}]%
        {leibo2021meltingpot}
\bibfield{author}{\bibinfo{person}{Joel~Z Leibo}, \bibinfo{person}{Edgar~A
  Due{\~n}ez-Guzman}, \bibinfo{person}{Alexander Vezhnevets},
  \bibinfo{person}{John~P Agapiou}, \bibinfo{person}{Peter Sunehag},
  \bibinfo{person}{Raphael Koster}, \bibinfo{person}{Jayd Matyas},
  \bibinfo{person}{Charlie Beattie}, \bibinfo{person}{Igor Mordatch}, {and}
  \bibinfo{person}{Thore Graepel}.} \bibinfo{year}{2021}\natexlab{}.
\newblock \showarticletitle{Scalable evaluation of multi-agent reinforcement
  learning with {M}elting {P}ot}. In \bibinfo{booktitle}{\emph{International
  Conference on Machine Learning}}. PMLR, \bibinfo{pages}{6187--6199}.
\newblock


\bibitem[\protect\citeauthoryear{Liu, Lanctot, Marris, and Heess}{Liu
  et~al\mbox{.}}{2022a}]%
        {liu2022simplex}
\bibfield{author}{\bibinfo{person}{Siqi Liu}, \bibinfo{person}{Marc Lanctot},
  \bibinfo{person}{Luke Marris}, {and} \bibinfo{person}{Nicolas Heess}.}
  \bibinfo{year}{2022}\natexlab{a}.
\newblock \showarticletitle{Simplex Neural Population Learning: Any-Mixture
  {B}ayes-Optimality in Symmetric Zero-sum Games}. In
  \bibinfo{booktitle}{\emph{Proceedings of the 39th International Conference on
  Machine Learning}} \emph{(\bibinfo{series}{Proceedings of Machine Learning
  Research}, Vol.~\bibinfo{volume}{162})},
  \bibfield{editor}{\bibinfo{person}{Kamalika Chaudhuri},
  \bibinfo{person}{Stefanie Jegelka}, \bibinfo{person}{Le~Song},
  \bibinfo{person}{Csaba Szepesvari}, \bibinfo{person}{Gang Niu}, {and}
  \bibinfo{person}{Sivan Sabato}} (Eds.). \bibinfo{publisher}{PMLR},
  \bibinfo{pages}{13793--13806}.
\newblock
\urldef\tempurl%
\url{https://proceedings.mlr.press/v162/liu22h.html}
\showURL{%
\tempurl}


\bibitem[\protect\citeauthoryear{Liu, Lever, Heess, Merel, Tunyasuvunakool, and
  Graepel}{Liu et~al\mbox{.}}{2019}]%
        {liu2019emergent}
\bibfield{author}{\bibinfo{person}{Siqi Liu}, \bibinfo{person}{Guy Lever},
  \bibinfo{person}{Nicholas Heess}, \bibinfo{person}{Josh Merel},
  \bibinfo{person}{Saran Tunyasuvunakool}, {and} \bibinfo{person}{Thore
  Graepel}.} \bibinfo{year}{2019}\natexlab{}.
\newblock \showarticletitle{Emergent Coordination Through Competition}. In
  \bibinfo{booktitle}{\emph{International Conference on Learning
  Representations}}.
\newblock
\urldef\tempurl%
\url{https://openreview.net/forum?id=BkG8sjR5Km}
\showURL{%
\tempurl}


\bibitem[\protect\citeauthoryear{Liu, Lever, Wang, Merel, Eslami, Hennes,
  Czarnecki, Tassa, Omidshafiei, Abdolmaleki, et~al\mbox{.}}{Liu
  et~al\mbox{.}}{2022b}]%
        {liu2022motor}
\bibfield{author}{\bibinfo{person}{Siqi Liu}, \bibinfo{person}{Guy Lever},
  \bibinfo{person}{Zhe Wang}, \bibinfo{person}{Josh Merel},
  \bibinfo{person}{SM~Ali Eslami}, \bibinfo{person}{Daniel Hennes},
  \bibinfo{person}{Wojciech~M Czarnecki}, \bibinfo{person}{Yuval Tassa},
  \bibinfo{person}{Shayegan Omidshafiei}, \bibinfo{person}{Abbas Abdolmaleki},
  {et~al\mbox{.}}} \bibinfo{year}{2022}\natexlab{b}.
\newblock \showarticletitle{From motor control to team play in simulated
  humanoid football}.
\newblock \bibinfo{journal}{\emph{Science Robotics}} \bibinfo{volume}{7},
  \bibinfo{number}{69} (\bibinfo{year}{2022}), \bibinfo{pages}{eabo0235}.
\newblock


\bibitem[\protect\citeauthoryear{Liu, Lever, Wang, Merel, Eslami, Hennes,
  Czarnecki, Tassa, Omidshafiei, Abdolmaleki, Siegel, Hasenclever, Marris,
  Tunyasuvunakool, Song, Wulfmeier, Muller, Haarnoja, Tracey, Tuyls, Graepel,
  and Heess}{Liu et~al\mbox{.}}{2021}]%
        {liu_motor_2021}
\bibfield{author}{\bibinfo{person}{Siqi Liu}, \bibinfo{person}{Guy Lever},
  \bibinfo{person}{Zhe Wang}, \bibinfo{person}{Josh Merel},
  \bibinfo{person}{S.~M.~Ali Eslami}, \bibinfo{person}{Daniel Hennes},
  \bibinfo{person}{Wojciech~M. Czarnecki}, \bibinfo{person}{Yuval Tassa},
  \bibinfo{person}{Shayegan Omidshafiei}, \bibinfo{person}{Abbas Abdolmaleki},
  \bibinfo{person}{Noah Siegel}, \bibinfo{person}{Leonard Hasenclever},
  \bibinfo{person}{Luke Marris}, \bibinfo{person}{Saran Tunyasuvunakool},
  \bibinfo{person}{H.~Francis Song}, \bibinfo{person}{Markus Wulfmeier},
  \bibinfo{person}{Paul Muller}, \bibinfo{person}{Tuomas Haarnoja},
  \bibinfo{person}{Brendan~D. Tracey}, \bibinfo{person}{Karl Tuyls},
  \bibinfo{person}{Thore Graepel}, {and} \bibinfo{person}{Nicolas Manfred~Otto
  Heess}.} \bibinfo{year}{2021}\natexlab{}.
\newblock \showarticletitle{From Motor Control to Team Play in Simulated
  Humanoid Football}.
\newblock \bibinfo{journal}{\emph{Science robotics}}  \bibinfo{volume}{7 69}
  (\bibinfo{year}{2021}), \bibinfo{pages}{eabo0235}.
\newblock
\urldef\tempurl%
\url{https://api.semanticscholar.org/CorpusID:235195692}
\showURL{%
\tempurl}


\bibitem[\protect\citeauthoryear{Liu, Marris, Hennes, Merel, Heess, and
  Graepel}{Liu et~al\mbox{.}}{2022c}]%
        {liu2022neupl}
\bibfield{author}{\bibinfo{person}{Siqi Liu}, \bibinfo{person}{Luke Marris},
  \bibinfo{person}{Daniel Hennes}, \bibinfo{person}{Josh Merel},
  \bibinfo{person}{Nicolas Heess}, {and} \bibinfo{person}{Thore Graepel}.}
  \bibinfo{year}{2022}\natexlab{c}.
\newblock \showarticletitle{Neu{PL}: Neural Population Learning}. In
  \bibinfo{booktitle}{\emph{International Conference on Learning
  Representations}}.
\newblock
\urldef\tempurl%
\url{https://openreview.net/forum?id=MIX3fJkl_1}
\showURL{%
\tempurl}


\bibitem[\protect\citeauthoryear{Marris, Gemp, Anthony, Tacchetti, Liu, and
  Tuyls}{Marris et~al\mbox{.}}{2022}]%
        {marris2022turbocharging}
\bibfield{author}{\bibinfo{person}{Luke Marris}, \bibinfo{person}{Ian Gemp},
  \bibinfo{person}{Thomas Anthony}, \bibinfo{person}{Andrea Tacchetti},
  \bibinfo{person}{Siqi Liu}, {and} \bibinfo{person}{Karl Tuyls}.}
  \bibinfo{year}{2022}\natexlab{}.
\newblock \showarticletitle{Turbocharging Solution Concepts: Solving {NE}s,
  {CE}s and {CCE}s with Neural Equilibrium Solvers}. In
  \bibinfo{booktitle}{\emph{Advances in Neural Information Processing
  Systems}}, \bibfield{editor}{\bibinfo{person}{Alice~H. Oh},
  \bibinfo{person}{Alekh Agarwal}, \bibinfo{person}{Danielle Belgrave}, {and}
  \bibinfo{person}{Kyunghyun Cho}} (Eds.).
\newblock
\urldef\tempurl%
\url{https://openreview.net/forum?id=RczPtvlaXPH}
\showURL{%
\tempurl}


\bibitem[\protect\citeauthoryear{Marris, Muller, Lanctot, Tuyls, and
  Graepel}{Marris et~al\mbox{.}}{2021}]%
        {marris2021jpsroicml}
\bibfield{author}{\bibinfo{person}{Luke Marris}, \bibinfo{person}{Paul Muller},
  \bibinfo{person}{Marc Lanctot}, \bibinfo{person}{Karl Tuyls}, {and}
  \bibinfo{person}{Thore Graepel}.} \bibinfo{year}{2021}\natexlab{}.
\newblock \showarticletitle{Multi-Agent Training beyond Zero-Sum with
  Correlated Equilibrium Meta-Solvers}. In
  \bibinfo{booktitle}{\emph{Proceedings of the 38th International Conference on
  Machine Learning}} \emph{(\bibinfo{series}{Proceedings of Machine Learning
  Research}, Vol.~\bibinfo{volume}{139})},
  \bibfield{editor}{\bibinfo{person}{Marina Meila} {and} \bibinfo{person}{Tong
  Zhang}} (Eds.). \bibinfo{publisher}{PMLR}, \bibinfo{pages}{7480--7491}.
\newblock
\urldef\tempurl%
\url{http://proceedings.mlr.press/v139/marris21a.html}
\showURL{%
\tempurl}


\bibitem[\protect\citeauthoryear{Mcaleer, Lanier, Fox, and Baldi}{Mcaleer
  et~al\mbox{.}}{2020}]%
        {mcaleer_pipeline_2021}
\bibfield{author}{\bibinfo{person}{Stephen Mcaleer}, \bibinfo{person}{JB
  Lanier}, \bibinfo{person}{Roy Fox}, {and} \bibinfo{person}{Pierre Baldi}.}
  \bibinfo{year}{2020}\natexlab{}.
\newblock \showarticletitle{Pipeline PSRO: A Scalable Approach for Finding
  Approximate Nash Equilibria in Large Games}. In
  \bibinfo{booktitle}{\emph{Advances in Neural Information Processing
  Systems}}, \bibfield{editor}{\bibinfo{person}{H.~Larochelle},
  \bibinfo{person}{M.~Ranzato}, \bibinfo{person}{R.~Hadsell},
  \bibinfo{person}{M.~F. Balcan}, {and} \bibinfo{person}{H.~Lin}} (Eds.),
  Vol.~\bibinfo{volume}{33}. \bibinfo{publisher}{Curran Associates, Inc.},
  \bibinfo{pages}{20238--20248}.
\newblock
\urldef\tempurl%
\url{https://proceedings.neurips.cc/paper/2020/file/e9bcd1b063077573285ae1a41025f5dc-Paper.pdf}
\showURL{%
\tempurl}


\bibitem[\protect\citeauthoryear{{McMahan}, Gordon, and Blum}{{McMahan}
  et~al\mbox{.}}{2003}]%
        {mcmahan_planning_2003}
\bibfield{author}{\bibinfo{person}{H.~Brendan {McMahan}},
  \bibinfo{person}{Geoffrey~J. Gordon}, {and} \bibinfo{person}{Avrim Blum}.}
  \bibinfo{year}{2003}\natexlab{}.
\newblock \showarticletitle{Planning in the presence of cost functions
  controlled by an adversary}. In \bibinfo{booktitle}{\emph{Proceedings of the
  Twentieth International Conference on International Conference on Machine
  Learning}} (Washington, {DC}, {USA}, 2003-08-21)
  \emph{(\bibinfo{series}{{ICML}'03})}. \bibinfo{publisher}{{AAAI} Press},
  \bibinfo{pages}{536--543}.
\newblock
\showISBNx{978-1-57735-189-4}


\bibitem[\protect\citeauthoryear{Moulin and Vial}{Moulin and Vial}{1978}]%
        {moulin1978strategically}
\bibfield{author}{\bibinfo{person}{Herv{\'e} Moulin} {and} \bibinfo{person}{J-P
  Vial}.} \bibinfo{year}{1978}\natexlab{}.
\newblock \showarticletitle{Strategically zero-sum games: the class of games
  whose completely mixed equilibria cannot be improved upon}.
\newblock \bibinfo{journal}{\emph{International Journal of Game Theory}}
  \bibinfo{volume}{7}, \bibinfo{number}{3} (\bibinfo{year}{1978}),
  \bibinfo{pages}{201--221}.
\newblock


\bibitem[\protect\citeauthoryear{Muller, Omidshafiei, Rowland, Tuyls, Perolat,
  Liu, Hennes, Marris, Lanctot, Hughes, Wang, Lever, Heess, Graepel, and
  Munos}{Muller et~al\mbox{.}}{2020}]%
        {Muller2020A}
\bibfield{author}{\bibinfo{person}{Paul Muller}, \bibinfo{person}{Shayegan
  Omidshafiei}, \bibinfo{person}{Mark Rowland}, \bibinfo{person}{Karl Tuyls},
  \bibinfo{person}{Julien Perolat}, \bibinfo{person}{Siqi Liu},
  \bibinfo{person}{Daniel Hennes}, \bibinfo{person}{Luke Marris},
  \bibinfo{person}{Marc Lanctot}, \bibinfo{person}{Edward Hughes},
  \bibinfo{person}{Zhe Wang}, \bibinfo{person}{Guy Lever},
  \bibinfo{person}{Nicolas Heess}, \bibinfo{person}{Thore Graepel}, {and}
  \bibinfo{person}{Remi Munos}.} \bibinfo{year}{2020}\natexlab{}.
\newblock \showarticletitle{A Generalized Training Approach for Multiagent
  Learning}. In \bibinfo{booktitle}{\emph{International Conference on Learning
  Representations}}.
\newblock
\urldef\tempurl%
\url{https://openreview.net/forum?id=Bkl5kxrKDr}
\showURL{%
\tempurl}


\bibitem[\protect\citeauthoryear{Munos, Stepleton, Harutyunyan, and
  Bellemare}{Munos et~al\mbox{.}}{2016}]%
        {munos2016safe}
\bibfield{author}{\bibinfo{person}{R{\'e}mi Munos}, \bibinfo{person}{Tom
  Stepleton}, \bibinfo{person}{Anna Harutyunyan}, {and} \bibinfo{person}{Marc
  Bellemare}.} \bibinfo{year}{2016}\natexlab{}.
\newblock \showarticletitle{Safe and efficient off-policy reinforcement
  learning}.
\newblock \bibinfo{journal}{\emph{Advances in neural information processing
  systems}}  \bibinfo{volume}{29} (\bibinfo{year}{2016}).
\newblock


\bibitem[\protect\citeauthoryear{Nash}{Nash}{1951}]%
        {nash1951}
\bibfield{author}{\bibinfo{person}{John Nash}.}
  \bibinfo{year}{1951}\natexlab{}.
\newblock \showarticletitle{Non-Cooperative Games}.
\newblock \bibinfo{journal}{\emph{Annals of Mathematics}} \bibinfo{volume}{54},
  \bibinfo{number}{2} (\bibinfo{year}{1951}), \bibinfo{pages}{286--295}.
\newblock
\showISSN{0003486X}
\urldef\tempurl%
\url{http://www.jstor.org/stable/1969529}
\showURL{%
\tempurl}


\bibitem[\protect\citeauthoryear{Omidshafiei, Papadimitriou, Piliouras, Tuyls,
  Rowland, Lespiau, Czarnecki, Lanctot, Perolat, and Munos}{Omidshafiei
  et~al\mbox{.}}{2019}]%
        {omidshafiei2019alpha}
\bibfield{author}{\bibinfo{person}{Shayegan Omidshafiei},
  \bibinfo{person}{Christos Papadimitriou}, \bibinfo{person}{Georgios
  Piliouras}, \bibinfo{person}{Karl Tuyls}, \bibinfo{person}{Mark Rowland},
  \bibinfo{person}{Jean-Baptiste Lespiau}, \bibinfo{person}{Wojciech~M
  Czarnecki}, \bibinfo{person}{Marc Lanctot}, \bibinfo{person}{Julien Perolat},
  {and} \bibinfo{person}{Remi Munos}.} \bibinfo{year}{2019}\natexlab{}.
\newblock \showarticletitle{$\alpha$-rank: Multi-agent evaluation by
  evolution}.
\newblock \bibinfo{journal}{\emph{Scientific reports}} \bibinfo{volume}{9},
  \bibinfo{number}{1} (\bibinfo{year}{2019}), \bibinfo{pages}{1--29}.
\newblock


\bibitem[\protect\citeauthoryear{Perez, Strub, De~Vries, Dumoulin, and
  Courville}{Perez et~al\mbox{.}}{2018}]%
        {perez2018film}
\bibfield{author}{\bibinfo{person}{Ethan Perez}, \bibinfo{person}{Florian
  Strub}, \bibinfo{person}{Harm De~Vries}, \bibinfo{person}{Vincent Dumoulin},
  {and} \bibinfo{person}{Aaron Courville}.} \bibinfo{year}{2018}\natexlab{}.
\newblock \showarticletitle{Film: Visual reasoning with a general conditioning
  layer}. In \bibinfo{booktitle}{\emph{Proceedings of the AAAI Conference on
  Artificial Intelligence}}, Vol.~\bibinfo{volume}{32}.
\newblock


\bibitem[\protect\citeauthoryear{Perolat, De~Vylder, Hennes, Tarassov, Strub,
  de~Boer, Muller, Connor, Burch, Anthony, et~al\mbox{.}}{Perolat
  et~al\mbox{.}}{2022}]%
        {perolat2022mastering}
\bibfield{author}{\bibinfo{person}{Julien Perolat}, \bibinfo{person}{Bart
  De~Vylder}, \bibinfo{person}{Daniel Hennes}, \bibinfo{person}{Eugene
  Tarassov}, \bibinfo{person}{Florian Strub}, \bibinfo{person}{Vincent de
  Boer}, \bibinfo{person}{Paul Muller}, \bibinfo{person}{Jerome~T Connor},
  \bibinfo{person}{Neil Burch}, \bibinfo{person}{Thomas Anthony},
  {et~al\mbox{.}}} \bibinfo{year}{2022}\natexlab{}.
\newblock \showarticletitle{Mastering the game of Stratego with model-free
  multiagent reinforcement learning}.
\newblock \bibinfo{journal}{\emph{Science}} \bibinfo{volume}{378},
  \bibinfo{number}{6623} (\bibinfo{year}{2022}), \bibinfo{pages}{990--996}.
\newblock


\bibitem[\protect\citeauthoryear{Perolat, Munos, Lespiau, Omidshafiei, Rowland,
  Ortega, Burch, Anthony, Balduzzi, De~Vylder, et~al\mbox{.}}{Perolat
  et~al\mbox{.}}{2021}]%
        {perolat2021poincare}
\bibfield{author}{\bibinfo{person}{Julien Perolat}, \bibinfo{person}{Remi
  Munos}, \bibinfo{person}{Jean-Baptiste Lespiau}, \bibinfo{person}{Shayegan
  Omidshafiei}, \bibinfo{person}{Mark Rowland}, \bibinfo{person}{Pedro Ortega},
  \bibinfo{person}{Neil Burch}, \bibinfo{person}{Thomas Anthony},
  \bibinfo{person}{David Balduzzi}, \bibinfo{person}{Bart De~Vylder},
  {et~al\mbox{.}}} \bibinfo{year}{2021}\natexlab{}.
\newblock \showarticletitle{From Poincar{\'e} recurrence to convergence in
  imperfect information games: Finding equilibrium via regularization}. In
  \bibinfo{booktitle}{\emph{International Conference on Machine Learning}}.
  PMLR, \bibinfo{pages}{8525--8535}.
\newblock


\bibitem[\protect\citeauthoryear{Ross}{Ross}{1971}]%
        {ross1971goofspiel}
\bibfield{author}{\bibinfo{person}{Sheldon~M Ross}.}
  \bibinfo{year}{1971}\natexlab{}.
\newblock \showarticletitle{Goofspiel—the game of pure strategy}.
\newblock \bibinfo{journal}{\emph{Journal of Applied Probability}}
  \bibinfo{volume}{8}, \bibinfo{number}{3} (\bibinfo{year}{1971}),
  \bibinfo{pages}{621--625}.
\newblock


\bibitem[\protect\citeauthoryear{Samuel}{Samuel}{1967}]%
        {samuel1967some}
\bibfield{author}{\bibinfo{person}{Arthur~L Samuel}.}
  \bibinfo{year}{1967}\natexlab{}.
\newblock \showarticletitle{Some studies in machine learning using the game of
  checkers. II—Recent progress}.
\newblock \bibinfo{journal}{\emph{IBM Journal of research and development}}
  \bibinfo{volume}{11}, \bibinfo{number}{6} (\bibinfo{year}{1967}),
  \bibinfo{pages}{601--617}.
\newblock


\bibitem[\protect\citeauthoryear{Schwarz, Czarnecki, Luketina,
  Grabska-Barwinska, Teh, Pascanu, and Hadsell}{Schwarz et~al\mbox{.}}{2018}]%
        {schwarz2018progress}
\bibfield{author}{\bibinfo{person}{Jonathan Schwarz}, \bibinfo{person}{Wojciech
  Czarnecki}, \bibinfo{person}{Jelena Luketina}, \bibinfo{person}{Agnieszka
  Grabska-Barwinska}, \bibinfo{person}{Yee~Whye Teh}, \bibinfo{person}{Razvan
  Pascanu}, {and} \bibinfo{person}{Raia Hadsell}.}
  \bibinfo{year}{2018}\natexlab{}.
\newblock \showarticletitle{Progress \& compress: A scalable framework for
  continual learning}. In \bibinfo{booktitle}{\emph{International Conference on
  Machine Learning}}. PMLR, \bibinfo{pages}{4528--4537}.
\newblock


\bibitem[\protect\citeauthoryear{Shahriari, Abdolmaleki, Byravan, Friesen, Liu,
  Springenberg, Heess, Hoffman, and Riedmiller}{Shahriari
  et~al\mbox{.}}{2022}]%
        {shahriari2022revisiting}
\bibfield{author}{\bibinfo{person}{Bobak Shahriari}, \bibinfo{person}{Abbas
  Abdolmaleki}, \bibinfo{person}{Arunkumar Byravan}, \bibinfo{person}{Abe
  Friesen}, \bibinfo{person}{Siqi Liu}, \bibinfo{person}{Jost~Tobias
  Springenberg}, \bibinfo{person}{Nicolas Heess}, \bibinfo{person}{Matt
  Hoffman}, {and} \bibinfo{person}{Martin Riedmiller}.}
  \bibinfo{year}{2022}\natexlab{}.
\newblock \showarticletitle{Revisiting Gaussian mixture critics in off-policy
  reinforcement learning: a sample-based approach}.
\newblock \bibinfo{journal}{\emph{arXiv preprint arXiv:2204.10256}}
  (\bibinfo{year}{2022}).
\newblock


\bibitem[\protect\citeauthoryear{Silver, Hubert, Schrittwieser, Antonoglou,
  Lai, Guez, Lanctot, Sifre, Kumaran, Graepel, et~al\mbox{.}}{Silver
  et~al\mbox{.}}{2018}]%
        {silver2018general}
\bibfield{author}{\bibinfo{person}{David Silver}, \bibinfo{person}{Thomas
  Hubert}, \bibinfo{person}{Julian Schrittwieser}, \bibinfo{person}{Ioannis
  Antonoglou}, \bibinfo{person}{Matthew Lai}, \bibinfo{person}{Arthur Guez},
  \bibinfo{person}{Marc Lanctot}, \bibinfo{person}{Laurent Sifre},
  \bibinfo{person}{Dharshan Kumaran}, \bibinfo{person}{Thore Graepel},
  {et~al\mbox{.}}} \bibinfo{year}{2018}\natexlab{}.
\newblock \showarticletitle{A general reinforcement learning algorithm that
  masters chess, shogi, and Go through self-play}.
\newblock \bibinfo{journal}{\emph{Science}} \bibinfo{volume}{362},
  \bibinfo{number}{6419} (\bibinfo{year}{2018}), \bibinfo{pages}{1140--1144}.
\newblock


\bibitem[\protect\citeauthoryear{Silver, Schrittwieser, Simonyan, Antonoglou,
  Huang, Guez, Hubert, Baker, Lai, Bolton, et~al\mbox{.}}{Silver
  et~al\mbox{.}}{2017}]%
        {silver2017mastering}
\bibfield{author}{\bibinfo{person}{David Silver}, \bibinfo{person}{Julian
  Schrittwieser}, \bibinfo{person}{Karen Simonyan}, \bibinfo{person}{Ioannis
  Antonoglou}, \bibinfo{person}{Aja Huang}, \bibinfo{person}{Arthur Guez},
  \bibinfo{person}{Thomas Hubert}, \bibinfo{person}{Lucas Baker},
  \bibinfo{person}{Matthew Lai}, \bibinfo{person}{Adrian Bolton},
  {et~al\mbox{.}}} \bibinfo{year}{2017}\natexlab{}.
\newblock \showarticletitle{Mastering the game of go without human knowledge}.
\newblock \bibinfo{journal}{\emph{nature}} \bibinfo{volume}{550},
  \bibinfo{number}{7676} (\bibinfo{year}{2017}), \bibinfo{pages}{354--359}.
\newblock


\bibitem[\protect\citeauthoryear{Smith, Anthony, and Wellman}{Smith
  et~al\mbox{.}}{2020}]%
        {smith2020iterative}
\bibfield{author}{\bibinfo{person}{Max Smith}, \bibinfo{person}{Thomas
  Anthony}, {and} \bibinfo{person}{Michael Wellman}.}
  \bibinfo{year}{2020}\natexlab{}.
\newblock \showarticletitle{Iterative Empirical Game Solving via Single Policy
  Best Response}. In \bibinfo{booktitle}{\emph{International Conference on
  Learning Representations}}.
\newblock


\bibitem[\protect\citeauthoryear{Southey, Bowling, Larson, Piccione, Burch,
  Billings, and Rayner}{Southey et~al\mbox{.}}{2005}]%
        {southey2005bayes}
\bibfield{author}{\bibinfo{person}{Finnegan Southey}, \bibinfo{person}{Michael
  Bowling}, \bibinfo{person}{Bryce Larson}, \bibinfo{person}{Carmelo Piccione},
  \bibinfo{person}{Neil Burch}, \bibinfo{person}{Darse Billings}, {and}
  \bibinfo{person}{Chris Rayner}.} \bibinfo{year}{2005}\natexlab{}.
\newblock \showarticletitle{Bayes' bluff: opponent modelling in poker}. In
  \bibinfo{booktitle}{\emph{Proceedings of the Twenty-First Conference on
  Uncertainty in Artificial Intelligence}}. \bibinfo{pages}{550--558}.
\newblock


\bibitem[\protect\citeauthoryear{Szafron, Gibson, and Sturtevant}{Szafron
  et~al\mbox{.}}{2013}]%
        {szafron2013parameterized}
\bibfield{author}{\bibinfo{person}{Duane Szafron}, \bibinfo{person}{Richard~G
  Gibson}, {and} \bibinfo{person}{Nathan~R Sturtevant}.}
  \bibinfo{year}{2013}\natexlab{}.
\newblock \showarticletitle{A parameterized family of equilibrium profiles for
  three-player kuhn poker.}. In \bibinfo{booktitle}{\emph{AAMAS}},
  Vol.~\bibinfo{volume}{13}. \bibinfo{pages}{247--254}.
\newblock


\bibitem[\protect\citeauthoryear{Tassa, Tunyasuvunakool, Muldal, Doron, Liu,
  Bohez, Merel, Erez, Lillicrap, and Heess}{Tassa et~al\mbox{.}}{2020}]%
        {tassa2020dm_control}
\bibfield{author}{\bibinfo{person}{Yuval Tassa}, \bibinfo{person}{Saran
  Tunyasuvunakool}, \bibinfo{person}{Alistair Muldal}, \bibinfo{person}{Yotam
  Doron}, \bibinfo{person}{Siqi Liu}, \bibinfo{person}{Steven Bohez},
  \bibinfo{person}{Josh Merel}, \bibinfo{person}{Tom Erez},
  \bibinfo{person}{Timothy~P. Lillicrap}, {and} \bibinfo{person}{Nicolas
  Manfred~Otto Heess}.} \bibinfo{year}{2020}\natexlab{}.
\newblock \showarticletitle{dm\_control: Software and Tasks for Continuous
  Control}.
\newblock \bibinfo{journal}{\emph{Softw. Impacts}}  \bibinfo{volume}{6}
  (\bibinfo{year}{2020}), \bibinfo{pages}{100022}.
\newblock
\urldef\tempurl%
\url{https://api.semanticscholar.org/CorpusID:219980295}
\showURL{%
\tempurl}


\bibitem[\protect\citeauthoryear{Tesauro et~al\mbox{.}}{Tesauro
  et~al\mbox{.}}{1995}]%
        {tesauro1995temporal}
\bibfield{author}{\bibinfo{person}{Gerald Tesauro} {et~al\mbox{.}}}
  \bibinfo{year}{1995}\natexlab{}.
\newblock \showarticletitle{Temporal difference learning and TD-Gammon}.
\newblock \bibinfo{journal}{\emph{Commun. ACM}} \bibinfo{volume}{38},
  \bibinfo{number}{3} (\bibinfo{year}{1995}), \bibinfo{pages}{58--68}.
\newblock


\bibitem[\protect\citeauthoryear{Vinyals, Babuschkin, Czarnecki, Mathieu,
  Dudzik, Chung, Choi, Powell, Ewalds, Georgiev, et~al\mbox{.}}{Vinyals
  et~al\mbox{.}}{2019}]%
        {vinyals2019grandmaster}
\bibfield{author}{\bibinfo{person}{Oriol Vinyals}, \bibinfo{person}{Igor
  Babuschkin}, \bibinfo{person}{Wojciech~M Czarnecki},
  \bibinfo{person}{Micha{\"e}l Mathieu}, \bibinfo{person}{Andrew Dudzik},
  \bibinfo{person}{Junyoung Chung}, \bibinfo{person}{David~H Choi},
  \bibinfo{person}{Richard Powell}, \bibinfo{person}{Timo Ewalds},
  \bibinfo{person}{Petko Georgiev}, {et~al\mbox{.}}}
  \bibinfo{year}{2019}\natexlab{}.
\newblock \showarticletitle{Grandmaster level in StarCraft II using multi-agent
  reinforcement learning}.
\newblock \bibinfo{journal}{\emph{Nature}} \bibinfo{volume}{575},
  \bibinfo{number}{7782} (\bibinfo{year}{2019}), \bibinfo{pages}{350--354}.
\newblock


\end{thebibliography}

\newpage
\appendix
\onecolumn

\renewcommand{\theequation}{S.\arabic{equation}}

\section{Convergence Guarantees to CCE}
\label{app:cce_convergence}

Deep RL algorithms often produce stochastic policies. For instance, policy gradient and MPO (used in this work) produce stochastic policies. This may appear problematic because \jpsro's convergence proof assumes the BR operator returns deterministic policies, of which there are a finite number and, in the worst case, must all be enumerated. Fortunately, BR operators can be encouraged to produce specific stochastic policies through regularization, such as maximum entropy. We prove that there are a finite number of such specific stochastic policies. This extends \jpsro to provably converge with certain stochastic BR operators. We therefore use maximum-entropy RL for best-respond learning in this work.

\begin{definition}[Unique Stochastic Policy Mapping]
    A unique mapping from a set of deterministic policies to a stochastic policy. The subset of stochastic policies reachable from such a mapping is called the unique stochastic policy set.
\end{definition}

Many such unique stochastic policy mappings can be defined. A popular choice is the maximum entropy mapping. Because there are a finite number of subsets of deterministic policies, there are also a finite number of unique stochastic policies.

\begin{lemma}[Finite Unique Stochastic Policies] \label{lemma:finite_unique_policies}
    There are a finite number of unique stochastic policies. If a player has $|\Pi_p|$ deterministic policies, they have at most $2^{|\Pi_p|} - 1$ unique stochastic policies.
\end{lemma}

\begin{proof}
    Given a finite nonempty set of deterministic policies. A unique stochastic policy mapping will uniquely map this set to a stochastic policy. Suppose that there are $|\Pi_p|$ deterministic policies, there are $2^{|\Pi_p|} - 1$ nonempty subsets of deterministic policies. Some of these subsets could map to the same stochastic policy. Therefore $2^{|\Pi_p|} - 1$ is a finite upper bound on the number of reachable stochastic policies. 
\end{proof}

If a best-response operator mixes uniquely amongst possible deterministic policies, the best-response operator can only produce a finite number of unique stochastic policies. The maximum entropy best-response operator is one such unique mixing. This is particularly useful when using RL as a best-response operator. Often RL algorithms, such as policy gradient, produce stochastic policies and and use entropy regularization to aid exploration.

\section{Methods}

\subsection{Representing best-responses to metagame mixed-strategies}
\label{app:representing_br}

As the number of co-player joint strategies under the marginal CCE distribution $\sigma_\notp$ grows exponentially in the number of co-players, care needs to be taken when implementing the conditional BR policy network $\Pi_\phi(\cdot | s, \sigma_\notp)$.

Our approach is the following. To represent $\sigma_\notp$, we consider the {\tt top-k} elements of the joint distribution $\sigma$ and for each strategy assignment $a = (a_1, \dots, a_n)$, return an aggregated strategy assignment representation using strategy embeddings from all players other than $\cV_p$. To offer a probabilistic interpretation of the marginal CCE representation, the final representation is a sum over strategy assignment representations, weighted by their probabilities under $\sigma$. This marginal CCE encoder $g$ is described in Equation~\ref{eq:marginal_encoder}. 

This implementation has several practical benefits. First, the memory footprint of the representation is constant as a function of the hyper-parameter $K$. This offers a trade-off between lower memory footprint and a lossless representation of the marginal distribution $\sigma_\notp$. Second, this representation preserves the probabilistic interpretation of $\sigma_\notp$. Co-player joint strategies that are more probable under $\sigma_\notp$ also features more prominently in the final representation. Lastly, this representation is amenable to hardware scaling, as marginal distributions from the perspectives of different players can be efficiently batch-processed with hardware acceleration.

\begin{equation} \label{eq:marginal_encoder}
  g(\cV, \sigma, p, K) = \sum_{a \in \texttt{top-k}(\sigma, K)} \sigma(a) f( \nu^{a_1}_1, \dots, \nu^{a_{p-1}}_{p-1}, \boldsymbol{0}, \nu^{a_{p+1}}_{p+1}, \dots, \nu^{a_n}_n )
\end{equation}

In all experiments reported in this work, we used $K = 96$. Empirically, we find that our representation of marginal CCE distributions is almost always lossless thanks to the inherent sparsity of the CCE solver we used. We show in Table~\ref{tab:top_k} the mean and standard deviation of the number of joint-actions with non-trivial support under the full CCE joint distributions in each game. We note that the co-player mixed-strategies we need to represent has one fewer dimension (action dimension of the best-responding player) than the full joint distribution.

\begin{table}[ht]
\centering
\caption{The number of joint actions with non-trivial support for each OpenSpiel game.\label{tab:top_k}}
\begin{tabular}{|l|c|c|c|}
\hline
Game                & \# joint actions (prob \textgreater 1e-3) & \# joint actions (prob \textgreater 5e-3) & \# joint actions (prob \textgreater 1e-2) \\ \hline
goofspiel\_2p\_5c   & 73.6 ± 42.2                               & 51.9 ± 25.3                               & 31.5 ± 12.7                               \\
kuhn\_poker\_2p     & 33.1 ± 22.6                               & 28.0 ± 17.1                               & 22.6 ± 11.0                               \\
kuhn\_poker\_3p     & 147.1 ± 127.2                             & 40.9 ± 28.4                               & 14.4 ± 14.7                               \\
leduc\_poker\_2p    & 24.0 ± 16.7                               & 20.8 ± 13.8                               & 17.8 ± 10.7                               \\
sheriff\_2p         & 42.9 ± 17.3                               & 34.2 ± 11.7                               & 24.4 ± 6.9                                \\
trade\_comm\_2p\_3i & 234.4 ± 158.5                             & 32.0 ± 31.3                               & 12.0 ± 14.6                               \\ \hline
\end{tabular}
\end{table}

The choice of strategy assignment aggregation function $f$ deserves attention too. If the game is asymmetric across all players, $f(\nu_1, \dots, \nu_n)$ can simply embed a concatenation of strategy embeddings from all players in order. If players take on symmetric roles, it maybe advantageous to leverage symmetry in the representation of strategy assignments. For instance, if player $i$ and player $j$ have symmetric roles in the game, we can set $\cV_i = \cV_j$ and apply an order-invariant pooling operator over the strategy embeddings of the two players such that $f(\dots, \nu^{a_i}_i, \dots, \nu^{a_j}_j, \dots) = f(\dots, \nu^{a_j}_i, \dots, \nu^{a_i}_j, \dots)$ by construction. We exploited symmetry in our experiments for symmetric games such as {\tt goofspiel} and {\tt capture-the-flag}.

\subsection{Policy optimisation}
\label{app:simulation}

We generate training episodes as follows. At iteration $t$ of best-response learning, we collect trajectories from policies sampled according to one of the joint CCE distributions $a = (a_1, \dots, a_n) \sim \sigma^\tau$ with $\tau < t$. With probability $\Pr_{br}(\tau, t)$, one of the players is randomly select to be the best-responding player of the episode. If player $p$ is best-responding, we execute its BR conditional policy network $\Pi_\phi(\cdot | s, \sigma^\tau_\notp)$. For all players $k$ that are not best-responding, $\Pi_{\hat\theta}(\cdot | s, {\hat\nu}^{a_k}_k)$ is used. All players learn from such trajectories. For player $p$, it uses this trajectory to optimise its RL objective. At the same time, the behaviours of its BR policy $\Pi_\phi(\cdot | s, \sigma^\tau_\notp)$ is distilled into $\Pi_\theta(\cdot | s, \nu^{\tau + 1}_p)$ by minimising the KL divergence between the two policies $\kld{\Pi_\theta(\cdot | s, \nu^{\tau + 1}_p)}{\Pi_\phi(\cdot | s, \sigma^\tau_\notp)}$. This corresponds to introducing player $p$'s latest strategy to the neural population.
For all non best-responding players $k$, their behaviours are regularised towards a stationary reference policy, by minimising the divergence between the current co-playing policies and the reference policies $\kld{\Pi_\theta(\cdot | s, \nu^{a_k}_k)}{\Pi_{\hat\theta}(\cdot | s, {\hat\nu}^{a_k}_k)}$. This regularisation procedure ensures that player $k$ would retain the behaviours of its previously acquired strategies over time. 

$\Pr_{br}(\tau, t)$ controls the balance between maintaining the stationarity of known policies and the incremental learning of new strategies. In our experiments, we used the following setting across all experiments, 
\[
    \Pr_{br}(\tau, t) = 
        \begin{cases}
            1 & \text{if } t = 1 \\
            \min(0.5, \max(0.2, t / T)) & \text{if } t > 1 \text{ and } \tau =  t - 1 \\
            0                           & \text{otherwise.}
        \end{cases}
\]

\subsection{Metagame solving}
\label{app:metagame_solving}

To solve for a metagame CCE distribution $\sigma^t \gets \textsc{MSS}(G^t)$, we need to evaluate all possible strategy combinations across $n$ players with $t \ge 1$ strategies each. Evaluating expected payoffs for each player across $t^n$ combinations requires a significant number of episodes to be simulated and it becomes increasingly costly as there are more players and more strategies. Instead, we turn to approximation and make use of a neural network with a vector-valued output $G(a) \gets \psi_w(\nu^{a_1}_1, \dots, \nu^{a_n}_n) \in \bR^n$ to approximate players' expected payoffs in a strategy assignment $a = (a_1, \dots, a_n)$, parameterised by $w$. Specifically, we took inspiration from \cite{liu2022neupl} and optimise $\psi_w(\nu^{a_1}_1, \dots, \nu^{a_n}_n)$ using the same value learning objective as in the underlying RL algorithm. Note that because $\psi_w$ is not conditioned on states, the expected returns is marginalised over state visitation distributions in episodes simulated with this specific strategy assignment. Further, in the same way that symmetry can be exploited in co-player joint strategy representation, order-invariant representations can be built into the payoff estimation network too which we exploit. Lastly, we note that the strategy assignment of {\em all} players should be treated as privileged information to the RL policies and value functions. Indeed, the metagame actions recommended to other players are private under the definition of a CCE (see Section~\ref{sec:cce}) and players should only be aware of the joint strategy profile which is public information to all players. This makes \neupljpsro a decentralised algorithm at test time.

For CCE solving, we make use of equilibrium solvers available from the open-source \href{https://github.com/deepmind/open_spiel/blob/master/open_spiel/python/examples/jpsro.py}{JPSRO implementation} of OpenSpiel \citep{LanctotEtAl2019OpenSpiel}. Throughout our experiments, we used the Max-Gini $\epsilon$-{\sc CCE} solver with $\epsilon = 0.01$ that we found necessary to ensure the numerical stability of the solver. It's worth noting that although CCE-solving is computationally tractable and can be solved using off-the-shelf LP solvers, we have found it to be a bottleneck when it comes to games with many players and many strategies. Recent works leveraging hardware acceleration and batched computation could be interesting to explore in this context \citep{marris2022turbocharging}.
To determine if a best-response iteration should terminate, we compare the expected value of the current best-response policies $\Pi_\phi(\cdot | s, \sigma^{t-1}_\notp)$ to the expected payoff estimate from $\sum_{a_\notp \sim \sigma^{t-1}_\notp} \Big[ \psi_w(\nu^{a_1}_1, \dots, \nu^t_p, \dots, \nu^{a_n}_n) \Big]$ against the same co-player mixed-strategy and verify that it is less than $\epsilon = 1\mathrm{e}{-3}$.

\section{Results}
\label{app:results}

We provide additional experimental details in this section. 

\paragraph{Optimisation configurations}
For best-response learning, we used the Maximum a Posteriori Policy Optimisation (MPO, \cite{abdolmaleki2018maximum}) algorithm with retrace \citep{munos2016safe} off-policy correction. For effective exploration and for empirical convergence (as discussed in Section~\ref{sec:scaling}), we maximise entropy as an auxiliary loss, following a decreasing linear schedule at the beginning of each best-response iteration. We set the M-step MPO epsilon to 0.01, E-step MPO epsilon\footnote{For continuous control, we used decoupled KL with an expected co-variance KL divergence of 1e-5.} to 5e-3 and the target network parameters are updated once every 200 gradient steps. These settings promote conservative gradient updates to the RL policies. We collect trajectory snippets of fixed length and make use of a replay buffer for off-policy learning. We compute and apply gradient updates from fixed sized batches of trajectories. For {\tt OpenSpiel/dm\_control/MeltingPot} domains, we used a trajectory length of 16/16/32, a batch size of 1024/2048/2048 and a replay buffer holding up to 2e5/1e6/1e6 trajectories respectively. We note that \neupljpsro is agnostic to the choice of deep RL algorithm and competitive alternatives exist in the literature. 
For optimisation, we used an ADAM optimiser \citep{KingmaB14} with a learning rate of 2e-4. Gradients are clipped by a maximum global norm of 10.0 for each update.

\paragraph{Observation encoder architecture}
In {\tt OpenSpiel} and {\tt dm\_control}, feature vectors are flattened and concatenated without additional processing. In {\tt MeltingPot}, a 2-layer convolutional neural network with output channels (16, 32), kernel shapes ((8, 8), (4, 4)) and strides (8, 1) is used, followed by a 2-layer MLP encoder of size (512, 512) whose outputs are concatenated with additional feature vectors from the game (e.g. flag status indicator). In all these domains, feature representation vectors are then encoded by a 3-layer MLP network of size (512, 256, 128). The memory architecture is feed-forward in all test domains other than {\tt MeltingPot} where players made use of a single-layer LSTM memory network with 256 hidden units to infer other players behaviours from their observation history of the game. Using the state representation $s$ from the memory network, we optimise the best-response policy head $\Pi_\phi(\cdot | s, \sigma_\notp)$ and the neural population of policies $\Pi_\theta(\cdot | s, \nu)$. FiLM multiplicative conditioning architecture \citep{perez2018film} is used to conditionally represent different policies in both cases. Equation~\ref{eq:marginal_encoder} and strategy embedding vectors $\cV$ are used to transform the state representation $s$ for $\Pi_\phi$ and $\Pi_\theta$ respectively.

\paragraph{Compute resources}
For {\tt OpenSpiel/dm\_control/MeltingPot} domains, we used 1 $\times$ V100 GPU/2 $\times$ V100 GPUs/1 $\times$ 8-core TPU-v4 to compute the gradient updates for all policies for all players respectively. For all domains, we used 384 independent CPU processes to simulate episodes and we used an additional accelerator to batch-process all policy inference requests.

\subsection{OpenSpiel}
\label{app:full_openspiel}

\begin{table}
    \centering
    \caption{OpenSpiel \citep{LanctotEtAl2019OpenSpiel} games investigated in this work where convergence to a CCE can be evaluated exactly.}
    \label{tab:games}
    \vskip 0.15in
    \begin{small}
    \begin{tabular}{|c|c|c|c|c|}
     \hline
     {\bf Game} & {\bf Players} & {\bf Payoffs} & {\bf Symmetry} & {\bf OpenSpiel Settings} \\
     \hline
     goofspiel\_2p\_5c & 2 & zero-sum & \cmark & 
\begin{lstlisting}
turn_based_simultaneous_game(
    game=goofspiel(
        egocentric=True, 
        imp_info=True, 
        num_cards=5, 
        num_turns=-1, 
        players=2, 
        points_order=descending, 
        returns_type=point_difference))
\end{lstlisting} \\
     \hline
     kuhn\_poker\_2p & 2 & zero-sum & \xmark & {\tt kuhn\_poker(players=2)} \\
     \hline
     kuhn\_poker\_3p & 3 & zero-sum & \xmark  & {\tt kuhn\_poker(players=3)} \\
     \hline
     leduc\_poker\_2p & 2 & zero-sum & \xmark  & {\tt leduc\_poker(players=2)} \\
     \hline
     sheriff\_2p & 2 & general-sum & \xmark  &
\begin{lstlisting}
sheriff(item_penalty=1.0, 
        item_value=5.0, 
        max_bribe=2, 
        max_items=10, 
        num_rounds=2, 
        sheriff_penalty=1.0)
\end{lstlisting} \\
     \hline
     trade\_comm\_2p\_3i & 2 & common & \xmark & {\tt trade\_comm(num\_items=3)} \\
     \hline
    \end{tabular}
    \end{small}
    \vskip -0.1in
\end{table}

We now provide additional details on our empirical results in OpenSpiel domains. Table~\ref{tab:games} summarises different properties of these games, including the number of players, the structure of the payoffs and if players take on symmetric roles in the game. Specific game settings are also provided and can be used to instantiate instances of the game using the open-source OpenSpiel loader \citep{LanctotEtAl2019OpenSpiel}. We note that our settings for ``sheriff'' is identical to that studied in \cite{farina2019correlation} which offers analytical results on the value of the maximum-welfare CCE in this game. For ``goofspiel'', our setting deviates from \cite{liu2022simplex} in two ways: 1) we used the standard game transformation from simultaneous game to its turn-based counterpart such that it offers the same interface as all other games; 2) we enabled ``egocentric'' observation such that players' observations are unaffected by its identity (i.e. player 1's actions are encoded as actions of the {\em other player} from the perspective of player 2). For all other games, we adopted identical settings as \cite{marris2021jpsroicml} for consistency.

Figure~\ref{fig:openspiel_per_player_cce} provides a detailed breakdown of CCE gaps and CCE values achieved by each player in each game, averaged over 5 seeds. In all games, we have observed empirical convergence to a CCE of the game for each player. Inspecting the CCE value of each player at convergence can reveal interests properties of the game. For instance, we have recovered the advantage to the last mover when players play their equilibrium strategies \citep{szafron2013parameterized, southey2005bayes, kuhn1950simplified}.

\begin{figure}
    \centering
    \includegraphics[width=\textwidth]{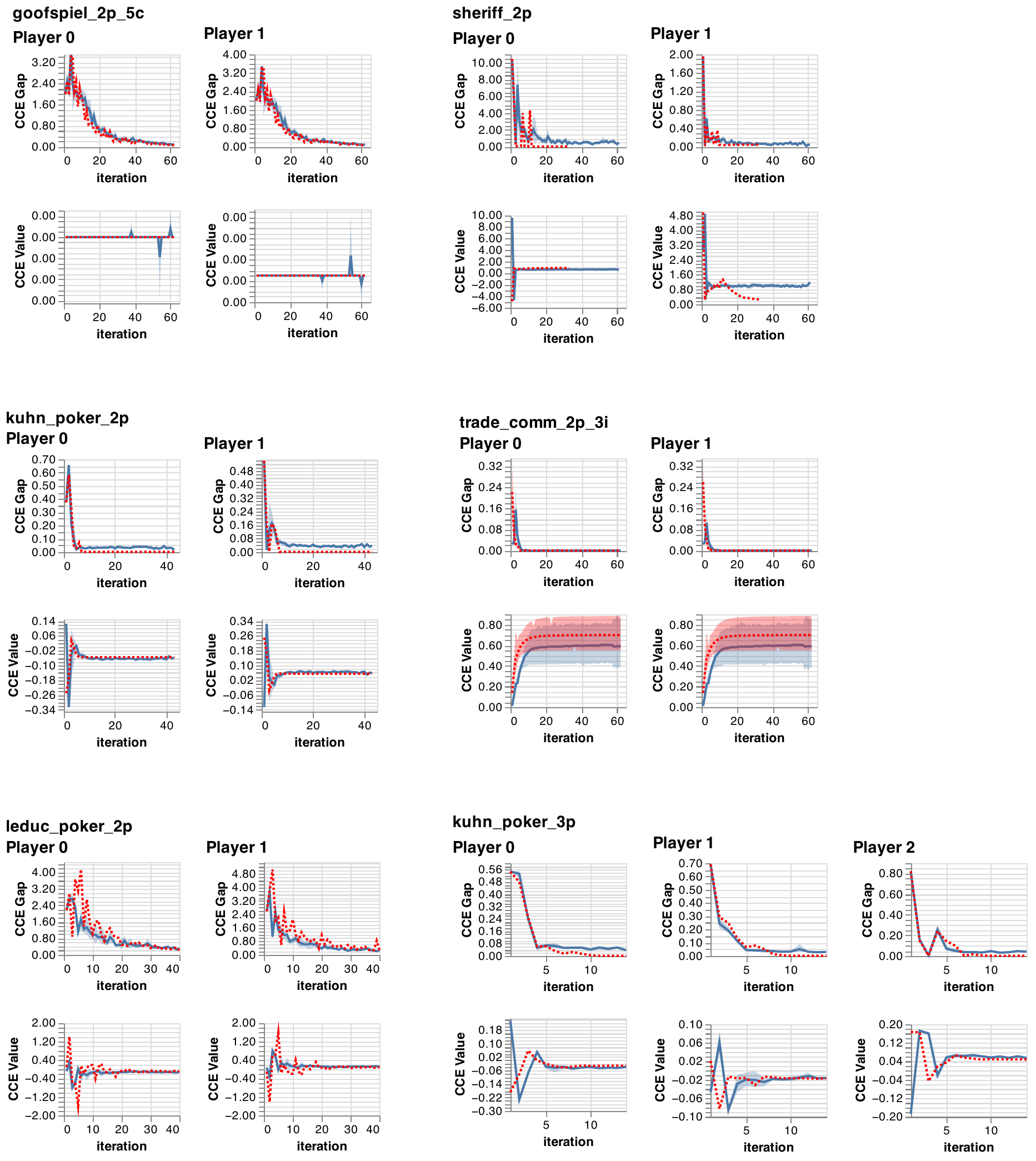}
    \caption{Exact CCE gaps and CCE values for each player averaged over 5 independent trials.\label{fig:openspiel_per_player_cce}}
\end{figure}

Figure~\ref{fig:fig_per_seed_012}-\ref{fig:fig_per_seed_345} provide a detailed view of the evolution of the sets of strategies represented by $\Pi_\theta$ over time. Each line shows the training progression of one best-responding policy. Each iteration starts with a dashed portion indicating that the best-response policy is still improving against its co-players and becomes solid once the best-responding iteration terminates and the next iteration begins. Brighter colours correspond to later iterations. The CCE gap and CCE value of iteration are computed analytically and monitored continuously in an experiment. We observe that in all games across all seeds, the neural network is capable of conditionally representing a large number of strategies for each player and maintain their behaviours over time reliably in most cases. In a few cases, we have observed that policies may change over time. This is reminiscent of the phenomenon of catastrophic forgetting observed in the continual learning literature and we leave to future works to investigate alternative continual learning techniques in this setting.

\begin{figure}
    \centering
    \includegraphics[width=0.9\textwidth,height=\textheight,keepaspectratio]{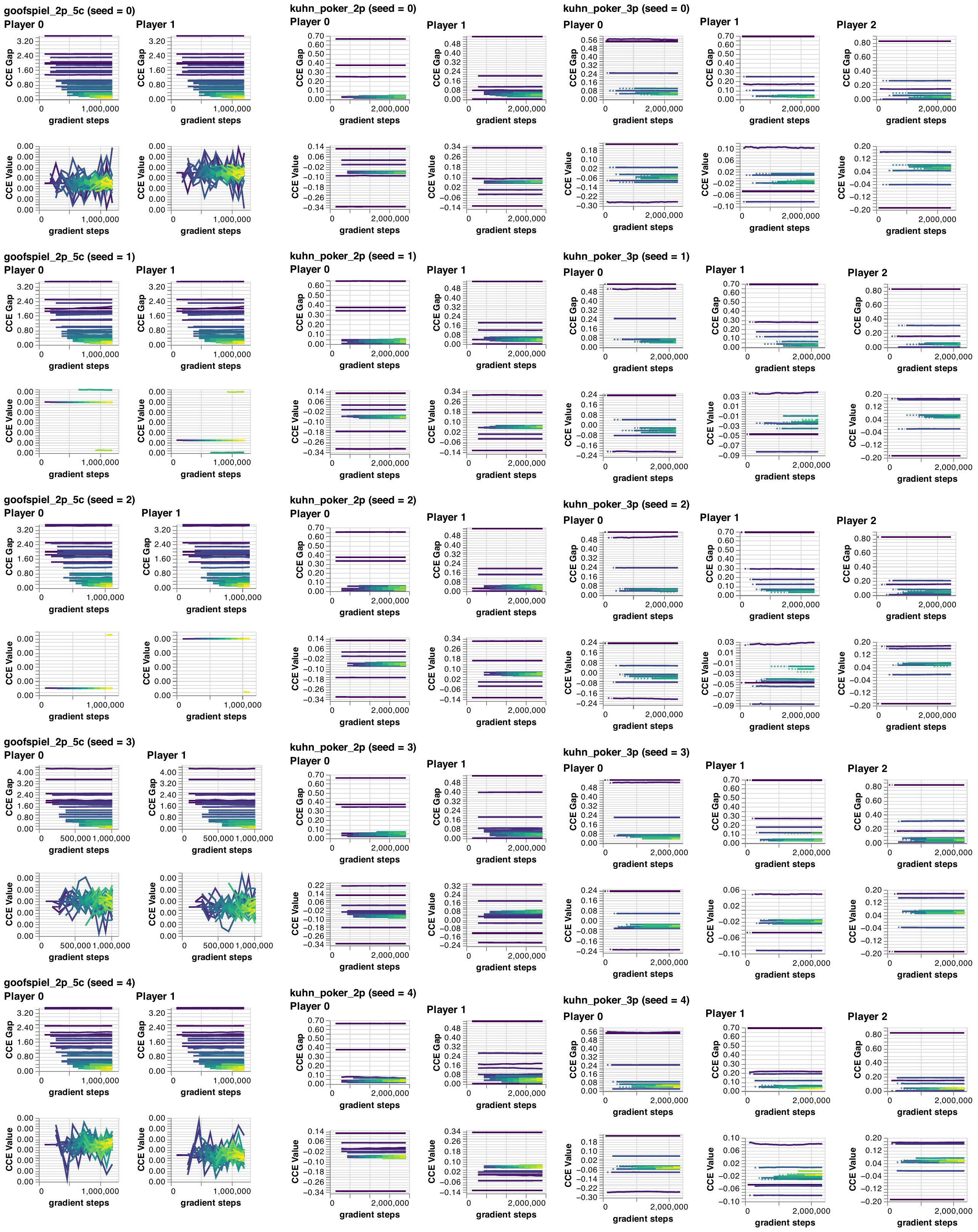}
    \caption{Emergence of strategies for each player in games. Each line corresponds to one policy represented in the neural population of policies. Brighter colours correspond to strategies discovered at later iterations.\label{fig:fig_per_seed_012}}
\end{figure}

\begin{figure}
    \centering
    \includegraphics[width=\textwidth,height=0.94\textheight,keepaspectratio]{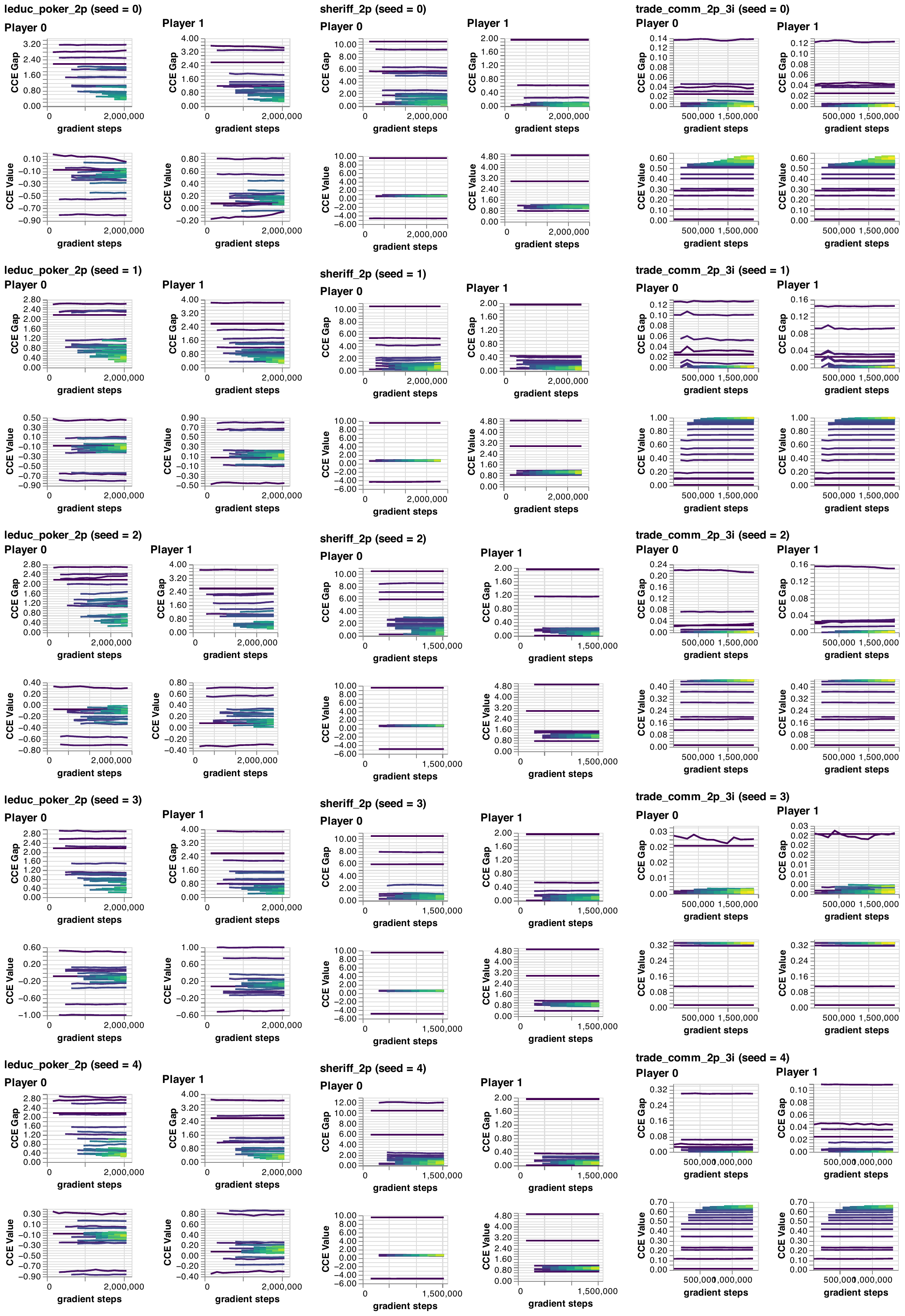}
    \caption{Emergence of strategies for each player in games. Each line corresponds to one policy represented in the neural population of policies. Brighter colours correspond to strategies discovered at later iterations (continued).\label{fig:fig_per_seed_345}}
\end{figure}

\subsection{Cheetah-run}

In our experiment, we used the standard continuous control benchmark task ``cheetah-run'' from the \href{github.com/deepmind/dm_control}{{\tt dm\_control}} task suite. To transform a single-agent control environment into a two-player common-payoff game, we implemented an environment wrapper that assigns each control actuator to one of the two players. A ``cheetah'' body has 6 joint actuators controlling different joints on the front and rear leg of the physical embodiment. Both players are rewarded with the same original task reward that provides a dense reward of 1.0 for each timestep when the forward velocity exceeds a predefined threshold. Finally, both players have access to the same set of observations, namely the positions and velocities of each of its joints. See \cite{tassa2020dm_control} for further details on the physical characteristics of this environment.

\subsection{Capture-the-Flag}
\label{app:ctf_results}

Figure~\ref{fig:ctf_payoffs} visualises the learning progression of \neupljpsro in 4-player {\tt capture-the-flag}. We observe that combinations of policies discovered at later iterations compete favorably against those of earlier iterations and receive higher payoffs. Figure~\ref{fig:ctf_cce_gap_average} additionally shows empirical convergence to a CCE of the game by observing the diminishing gain by switching to independently optimised best-response policies (Blue and Orange) compared to what the neural population already implements (Red). Comparing best-response policies learned {\em with} initial parameter transfer (Orange) to the ones {\em without} (Blue) reveals the benefits of transfer learning across strategies. The difference is more pronounced when best-responding to more sophisticated co-player policies.

\begin{figure}
    \centering
    \includegraphics[width=\textwidth]{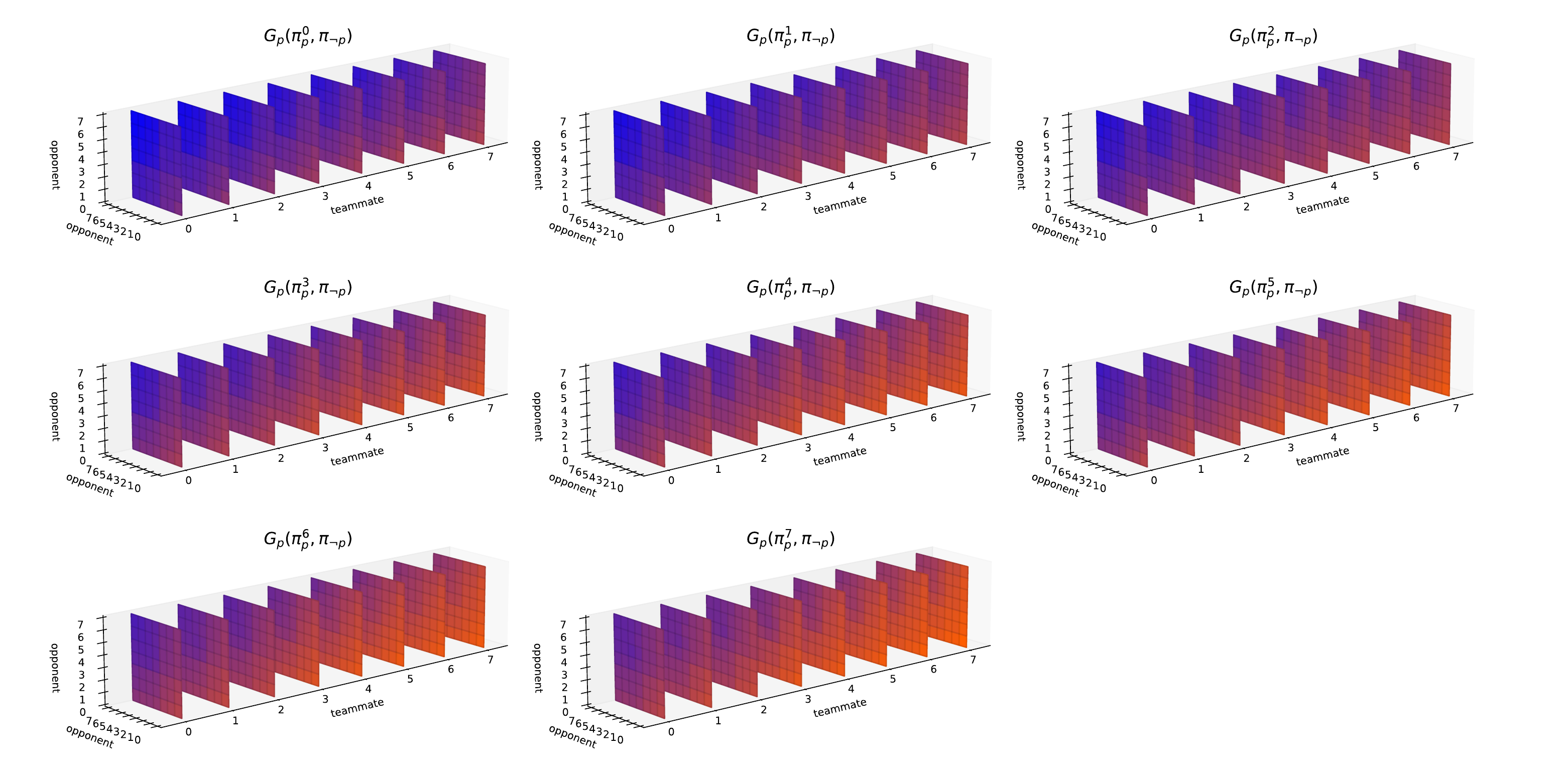}
    \caption{Expected returns for each policy of player $p$ in the presence of co-player policies. Orange (purple) indicates higher (lower) value.\label{fig:ctf_payoffs}}
\end{figure}

\begin{figure}
    \centering
    \includegraphics[width=\textwidth]{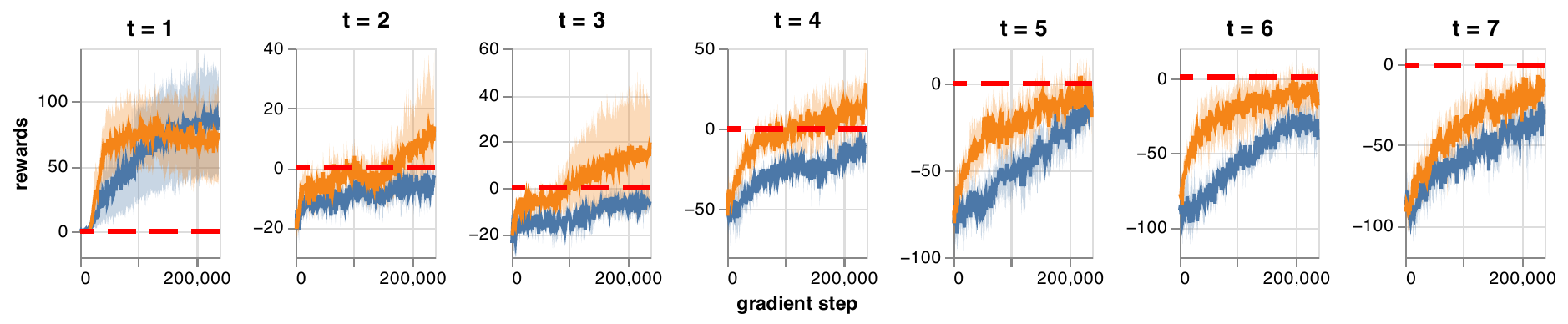}
    \caption{Empirical convergence to a CCE of the game as evidenced by the diminishing improvement offered by deviating to independently trained best-response policies. {\bf (Blue)} Learning progression of an independent RL policy for player $p$ when optimised against marginal CCE distributions at each iteration $\sigma^t_\notp$. {\bf (Orange)} Same as in Blue but the network parameters are partially initialised with that of the conditional policy $\Pi_\theta$. {\bf (Red)} Visualises the CCE value of player $p$ at each iteration when executing $\Pi_\theta$. Solid lines show the {\em average} rewards across 6 independent trials.\label{fig:ctf_cce_gap_average}}
\end{figure}

\end{document}